\documentclass[12pt]{colt2019-clean}
\pdfoutput=1

\usepackage[utf8]{inputenc}
\usepackage[T1]{fontenc}
\usepackage{hyperref}
\hypersetup{
        colorlinks   = true,
        urlcolor     = blue,
        linkcolor    = blue,
        citecolor   = blue
}

\usepackage{url}
\usepackage{booktabs}
\usepackage{amsmath,amsfonts,amssymb}
\usepackage{nicefrac}
\usepackage{microtype}
\usepackage{natbib}

\usepackage{enumitem}
\usepackage{algorithm}
\usepackage{algorithmicx}
\usepackage{algpseudocode}
\usepackage[nolist,nohyperlinks]{acronym}
\usepackage{color}
\usepackage{bold-extra}
\usepackage{mathtools}
\usepackage{etoolbox}
\usepackage{graphicx}
\usepackage{times}

\renewcommand{\P}{\mathbb{P}}
\newcommand{\Pgamma}{\mathbb{P}_{\gamma}}
\newcommand{\pigamma}{\pi_{\gamma,r}}
\newcommand{\piinfty}{\pi_{\infty}}
\newcommand{\qgamma}{q_{\gamma}}

\newcommand{\bA}{\boldsymbol{A}}

\newcommand{\bx}{\boldsymbol{x}}
\newcommand{\bxi}{\boldsymbol{\xi}}

\newcommand{\bM}{\boldsymbol{M}}
\newcommand{\btilM}{\boldsymbol{\tilde{M}}}

\newcommand{\bI}{\boldsymbol{I}}
\newcommand{\bu}{\boldsymbol{u}}

\newcommand{\sI}{\mathcal{I}}

\newcommand{\bz}{\boldsymbol{z}}

\newcommand{\bw}{\boldsymbol{w}}

\newcommand{\bhatw}{\wh{\boldsymbol{w}}}
\newcommand{\bhatwl}{\bhatw_{\lambda}}

\newcommand{\bv}{\boldsymbol{v}}
\newcommand{\bzero}{\boldsymbol{0}}

\newcommand{\sC}{\mathcal{C}}

\newcommand{\sN}{\mathcal{N}}
\newcommand{\sM}{\mathcal{M}}
\newcommand{\sB}{\mathcal{B}}

\newcommand{\sEstar}{\mathcal{E}^{\star}}

\newcommand{\sX}{\mathcal{X}}

\newcommand{\sE}{\mathcal{E}}

\newcommand{\sD}{\mathcal{D}}

\newcommand{\sZ}{\mathcal{Z}}

\DeclareMathOperator*{\argmin}{arg\,min}
\DeclareMathOperator*{\arginf}{arg\,inf}

\DeclareMathOperator*{\rank}{rank}

\newcommand{\field}[1]{\mathbb{#1}}

\newcommand{\R}{\field{R}}

\newcommand{\scO}{\mathcal{O}}

\newcommand{\sctilO}{\mathcal{\tilde{O}}}

\newcommand{\wh}{\widehat}
\newcommand{\ve}{\varepsilon}

\newtheorem{assumption}{Assumption}
\newtheorem{cor}{Corollary}
\newtheorem{prop}{Proposition}

\newcommand{\reals}{\mathbb{R}}

\newcommand{\tp}{^{\top}}
\newcommand{\Srep}{{S^{(i)}}}

\newcommand{\Rhat}{\wh{R}}
\newcommand{\Rhatl}{\wh{R}_{S,\lambda}}

\newcommand{\Rl}{R_{\lambda}}

\newcommand{\Rstarl}{R^{\star}_{\lambda}}
\newcommand{\Rstarli}{R^{\star}_{\lambda,i}}
\newcommand{\Rstarlj}{R^{\star}_{\lambda,j}}

\DeclareMathOperator*{\E}{\mathbb{E}}

\newcommand*\diff{\mathop{}\!\mathrm{d}}
\newcommand{\lmin}{\lambda_{\mathrm{min}}}

\newcommand{\lmini}{\lambda_{\mathrm{min},i}}

\newcommand{\tr}{\mathrm{tr}}
\newcommand{\KL}{\mathrm{KL}}

\newcommand{\pr}[1]{\left( #1 \right)}
\newcommand{\br}[1]{\left[ #1 \right]}
\newcommand{\cbr}[1]{\left\{ #1 \right\}}

\newcommand{\ps}{\wh{p}_{S,\gamma}}
\newcommand{\psng}{\wh{p}_{S}}
\newcommand{\psrep}{\wh{p}_{\Srep}}
\newcommand{\pscondA}{\wh{p}_{S \mid A}}
\newcommand{\psrepcondA}{\wh{p}_{\Srep \mid A}}
\newcommand{\pscond}{\wh{p}_{S \mid \sEstar}}
\newcommand{\qcond}{q_{\mid \sEstar}}

\newcommand{\leqC}{\lesssim}

\newcommand{\excess}{\Delta}

\newcommand{\glob}{^{\text{\scshape{glob}}}}

\newcommand{\subopt}{^{\text{\scshape{subopt}}}}

\newcommand{\polylog}{\text{polylog}}

\newcommand{\bwstar}{\bw^{\star}}
\newcommand{\bwstarl}{\bw^{\star}_{\lambda}}
\newcommand{\bwstarli}{\bw^{\star}_{\lambda, i}}
\newcommand{\bwstarlj}{\bw^{\star}_{\lambda, j}}
\newcommand{\bwstarlI}{\bw^{\star}_{\lambda, I}}
\newcommand{\bHstar}{\boldsymbol{H}^{\star}}
\newcommand{\bHstarl}{\boldsymbol{H}^{\star}_{\lambda}}
\newcommand{\bHstarli}{\boldsymbol{H}^{\star}_{\lambda,i}}
\newcommand{\bHstarI}{\boldsymbol{H}^{\star}_{I}}

\newcommand{\bHstarlj}{\boldsymbol{H}^{\star}_{\lambda,j}}
\newcommand{\bHstarlhalf}{\boldsymbol{H}^{\star \frac{1}{2}}_{\lambda}}
\newcommand{\bHstarlinv}{\boldsymbol{H}^{\star -1}_{\lambda}}
\newcommand{\bHstarlIinv}{\boldsymbol{H}^{\star -1}_{\lambda,I}}

\newcommand{\Lip}{L^{\star}}
\newcommand{\Lipi}{L^{\star}_i}

\newcommand{\mincomp}{\sC^{\star}}
\newcommand{\hypsp}{\reals^d}
\begin{acronym}
\acro{IC}{Improvement Condition}
\acro{RLS}{Regularized Least Squares}
\acro{TL}{Transfer Learning}
\acro{HTL}{Hypothesis Transfer Learning}
\acro{ERM}{Empirical Risk Minimization}
\acro{TEAM}{Target Empirical Accuracy Maximization}
\acro{RKHS}{Reproducing kernel Hilbert space}
\acro{DA}{Domain Adaptation}
\acro{LOO}{Leave-One-Out}
\acro{HP}{High Probability}
\acro{RSS}{Regularized Subset Selection}
\acro{FR}{Forward Regression}
\acro{PSD}{Positive Semi-Definite}
\acro{SGD}{Stochastic Gradient Descent}
\acro{OGD}{Online Gradient Descent}
\acro{EWA}{Exponentially Weighted Average}
\acro{EMD}{Effective Metric Dimension}
\acro{PDE}{Partial Differential Equation}
\acro{SDE}{Stochastic Differential Equation}
\acro{SGLD}{Stochastic Gradient Langevin Dynamics}
\end{acronym}

\title[Distribution-Dependent Analysis of Gibbs-ERM Principle]{Distribution-Dependent Analysis of Gibbs-ERM Principle}

\title{Distribution-Dependent Analysis of Gibbs-ERM Principle}

\coltauthor{%
 \Name{Ilja Kuzborskij } \Email{ilja.kuzborskij@gmail.com}\\
 \addr Universit\`{a} degli Studi di Milano
 \AND
 \Name{Nicol\`{o} Cesa-Bianchi} \Email{nicolo.cesa-bianchi@unimi.it}\\
 \addr Universit\`{a} degli Studi di Milano
 \AND
 \Name{Csaba Szepesv\'ari} \Email{szepi@google.com}\\
 \addr DeepMind, London
}

\begin{document}

\maketitle

% !TEX root =  main.tex
\begin{abstract}

Gibbs-ERM learning is a natural idealized model of learning with stochastic optimization algorithms (such as \acl{SGLD} and ---to some extent--- \acl{SGD}), while it also arises in other contexts, including PAC-Bayesian theory, and sampling mechanisms. In this work we study the excess risk suffered by a Gibbs-ERM learner that uses non-convex, regularized empirical risk with the goal to understand the interplay between the data-generating distribution and learning in large hypothesis spaces. 
Our main results are
  \emph{distribution-dependent} upper bounds on several notions of excess risk.
  We show that, in all cases, the distribution-dependent excess risk is essentially controlled by the \emph{effective dimension}
  $\tr\pr{\bHstar (\bHstar + \lambda \bI)^{-1}}$ of the problem,
  where $\bHstar$ is the Hessian matrix of the risk at a local minimum.
  This is a well-established notion of effective dimension appearing in several previous works, including the analyses of SGD and ridge regression, but ours is the first work that brings this dimension to the analysis of learning using Gibbs densities.
  The distribution-dependent view we advocate here improves upon earlier results of~\cite{raginsky2017nonconvex}, and can yield much tighter bounds depending on the interplay between the data-generating distribution and the loss function.
  The first part of our analysis focuses on the \emph{localized} excess risk in the vicinity of a fixed local minimizer.
  This result is then extended to bounds on the \emph{global} excess risk, by characterizing probabilities of local minima (and their complement) under Gibbs densities, a results which might be of independent interest.
\end{abstract}

% !TEX root =  main.tex
\section{Introduction}

In the parametric setting of statistical learning, the learner is given a tuple $S = \pr{z_1, \ldots, z_m}$ of \emph{training examples}, that are drawn independently of each other from a fixed and unknown probability distribution $\sD$ supported on an \emph{example space} $\sZ$. Based on the training examples $S$ the learner selects a model $\bw$ from a parameter space $\hypsp$. 
The learner's goal is to minimize the \emph{statistical risk} $R(\bw) = \E_z[\ell(\bw, z)]$ of the selected model, where $z$ is drawn from $\sD$ and $\ell \,:\, \hypsp \times \sZ \to [0, M]$ is some known \emph{loss function}, which we assume to be non-negative, bounded, and twice differentiable.

A learner following the \acf{ERM} principle selects a model with the smallest empirical risk. Learners often also incorporate a penalty, leading to selecting a model from the set
\begin{equation}
  \label{eq:erm}
  \argmin_{\bw \in \hypsp}\cbr{ \Rhat_S(\bw) + \lambda \|\bw\|^2}\,, \qquad \lambda > 0\,,
\end{equation}
where $\Rhat_S(\bw)$ is the \emph{empirical risk} of $\bw$, defined by
\[
  \Rhat_S(\bw) = \frac{1}{m} \sum_{i=1}^m \ell(\bw, z_i)\,.
\]
In the following, we will abbreviate the regularized empirical risk by $\Rhatl$ and its population counterpart, $\E[\Rhatl]$,  by $\Rl$.
In this paper we study a randomized version of \ac{ERM} known as \emph{Gibbs-ERM}.
A Gibbs algorithm outputs a model $\bw \in \hypsp$ sampled from the \emph{Gibbs density}
\begin{equation}
  \label{eq:gibbs}
  \ps(\bw) = \frac{1}{Z} \, e^{-\gamma \pr{ \Rhat_S(\bw) + \lambda \|\bw\|^2 }}\,, \qquad \gamma > 0\,,
\end{equation}
where $Z = \int_{\hypsp} e^{-\gamma \pr{ \Rhat_S(\bw) + \lambda \|\bw\|^2 }} \diff \bw$ is the normalization constant and $\Rhat_S$ is assumed to be such that $Z < \infty$ (for instance, this is the case when $\|\cdot\|$ is any norm and $\Rhat_S$ is nonnegative). 
It is not hard to see that we obtain \ac{ERM} as a special case of~\eqref{eq:gibbs} for $\gamma \to \infty$.
In the following expectations $\E[\cdot]$ are taken with respect to the joint distribution over the sample space $\sZ^m \times \reals^d$ (i.e., the product of the example space and the parameter space) unless explicitly stated otherwise, for instance $\E_{\bw \sim \ps}[\cdot]$.

Gibbs-ERM reveals its usefulness when the regularized empirical risk is non-convex and~\eqref{eq:erm} becomes intractable. This scenario brings out the connections between Gibbs-ERM and stochastic optimization algorithms, for instance \ac{SGLD} ---see below, along with a number of other settings in which Gibbs-ERM arises naturally.
One tantalizing related line of research lies in understanding theoretical properties of learning in overparameterized problems, such as deep neural networks, through the prism of stochastic optimization, since in these settings \ac{SGD} and its variants become de facto method of choice.
We believe that Gibbs-ERM principle provides an opportunity for explaining some of the learning-theoretic phenomena in this area.

In this paper we focus on the statistical properties of the Gibbs-ERM by analyzing distribution-dependent excess risk bounds.
In particular, we give upper bounds on the excess risk that can be much smaller compared to the previous literature, for instance \ac{SGLD}~\citep{raginsky2017nonconvex}, depending on the interplay between the data-generating distribution and the loss function.

\paragraph{\emph{\ac{SGLD} algorithm.}}
The recent interest in stochastic gradient descent algorithms for non-convex optimization led to the study of a variant called \ac{SGLD}. Apart from its simplicity, \ac{SGLD} has amenable theoretical properties, such as asymptotic convergence to global minima and polynomial saddle-point escape times~\citep{ge2015escaping}. 
The update rule of plain \ac{SGLD} is
\[
  \bhatw_{ t + 1 } = \bhatw_{ t } - \eta \nabla \Rhatl(\bhatw_t) + \sqrt { \frac{2 \eta}{\gamma} } \, \bxi_t \qquad t=0, 1, 2, \ldots
\]
where $\bhatw_1$ is sampled from a fixed distribution, $\bxi_t$ is a standard Gaussian ``noise'' vector (independent from the choice of $\bhatw_1$), and $\eta$ is a step size.
The \ac{SGLD} algorithm is known to approximate the continuous-time Langevin diffusion equation
\begin{equation}
  \label{eq:langevin}
  \diff \bw(t) = - \nabla \Rhatl(\bw(t)) \diff t + \sqrt { \frac{2}{\gamma} } \, \diff \boldsymbol{b}(t) , \quad t \geq 0
\end{equation}
where $\boldsymbol{b}(t)$ is the standard Brownian motion.
Indeed, under appropriate assumptions on the empirical risk, one can show that the solution to~\eqref{eq:langevin} admits~\eqref{eq:gibbs} as a stationary distribution~\citep{raginsky2017nonconvex,tzen2018local}.
While convergence in the limit is reassuring, the best known bound on the mixing time is of order $\polylog(1/\ve) \, e^{\sctilO(d)}$ for non-convex empirical risks assuming that objective is smooth and dissipative (roughly speaking, assuring that the process~\eqref{eq:langevin} on average moves towards the origin)~\citep{raginsky2017nonconvex,xu2018global}, 
and it is not clear whether the exponential dependence on the parameters can be eliminated.

\paragraph{\emph{Information Risk Minimization.}}
Gibbs-ERM naturally arises when introducing a relative entropy regularization in the so-called Information Risk Minimization framework~\citep{zhang2006information,xu2017information}. Indeed $\ps$ in~\eqref{eq:gibbs} can be equivalently defined as the solution to the following convex optimization problem
\begin{equation}
  \label{eq:IRM}
  \arginf_{\wh{p} \in \sM_1}\cbr{ \E_{\bw \sim \wh{p}}\br{\Rhat_S(\bw)} + \frac{1}{\gamma} \, \KL\pr{\wh{p} \ || \ \sN(\bzero, \lambda^{-1} \bI)} }\,,
\end{equation}
where $\sM_1$ is the set of all sample-dependent probability densities on $\hypsp$, with sample drawn from $\sD$, and $\KL$-divergence is defined between densities that are absolutely continuous with respect to some measure over $\hypsp$.
Problem~\eqref{eq:IRM} can be also motivated from perspective of the PAC-Bayesian analysis~\citep{mcallester1998some,seeger2002pac}, where~\eqref{eq:gibbs} is the density minimizing the bound on the expected risk.
Another instance of~\eqref{eq:IRM} is the well-known Maximum Entropy Discrimination framework of~\cite{jaakkola1999maximum}.
\paragraph{\emph{Sampling from~\eqref{eq:gibbs}.}}
Markov chain Monte Carlo (MCMC) algorithms can be used to sample directly from~\eqref{eq:gibbs}. Unfortunately, this is often known to be computationally inefficient~\citep{andrieu2003introduction}. On the other hand, there is a number of cases where MCMC demonstrates amenable computational properties, for instance when sampling from log-concave densities~\citep{andrieu2003introduction}.
Recent works have also showed that for a particular class of densities (such as smooth and strongly concave densities) variations of MCMC algorithms can sometimes achieve linear convergence~\citep{cheng2018underdamped},
for certain non-log-concave densities MCMC variants can achieve polynomial convergence in the dimension~\citep{cheng2018sharp}, or ever achieve faster convergence than optimization algorithms~\citep{ma2018sampling}.
Another popular line of research is a variational approximations of the Gibbs density, such as Variational Bayes~\citep{wang2018frequentist} where one resorts to the variational approximation of a target density.

\subsection{Our Contribution}
The algorithms discussed above perform randomized empirical risk minimization. 
However, minimizing empirical risk does not always lead to minimization of the risk. 
Hence, the quality of the solution $A(S)$ generated by the randomized algorithm $A$ given the training set $S$
is typically analyzed through the notion of \emph{excess risk} $\E_{S, A}\br{R(A(S))} - R(\bwstar)$, where $\bwstar$ is one of the minimizers of the risk. This is decomposed into the generalization error $R(A(S)) - \Rhat_S(A(S))$ and the term $\Rhat_S(A(S)) - R(\bwstar)$.
Similarly to~\cite{raginsky2017nonconvex}, we follow instead a Gibbs-centric decomposition of the excess risk
\[
  \E_{S, A}\br{R(A(S))} - R(\bwstar)
  =
  \underbrace{\E_{S, A, \bw \sim \ps}\br{R(A(S)) - R(\bw)}}_{\text{Computational excess risk}}
  +
  \underbrace{\E_{\bw \sim \ps}\br{R(\bw)} - R(\bwstar)}_{\text{Statistical excess risk}}\,.
\]
The first term is due solely to the dynamics of the algorithm, be it \ac{SGLD} or a sampling procedure, while the second one is a purely learning-theoretic quantity.
\cite{raginsky2017nonconvex} mainly focused on the finite-time analysis of the first term for \ac{SGLD} while showing convergence for non-convex objective functions.
Their analysis of the second term ---the statistical excess risk--- provides a bound of order (ignoring logarithmic factors)
\begin{equation}
  \label{eq:raginsky_excess}
  \frac{(\gamma + d)^2}{\lambda_{\star} m} + \frac{d}{\gamma}
\end{equation}
where $\lambda_{\star}$ is a positive spectral gap characterizing the exponential convergence rate of the Langevin diffusion to the stationary point. They conservatively bounded the reciprocal of $\lambda_{\star}$ as
\begin{equation}
  \label{eq:raginsky_spectral_gap}
  \frac{1}{\lambda_{\star}} = \sctilO \pr{ \frac { 1 } { \gamma ( d + \gamma ) } } + \pr{1 + \frac { d } { \gamma } } e ^ { \sctilO ( \gamma + d ) }~.
\end{equation}

This results in a statistical excess risk bound with a rather pessimistic exponential dependence on the ambient dimension $d$.
Therefore, a natural question to ask is whether the dependence on $d$ can be improved by taking into account specific properties of the learning problem,
and whether the dependence on $\lambda_{\star}$ can be avoided altogether (since the \emph{statistical} excess risk does not really depend on the convergence properties of \ac{SGLD}).
We believe that $\lambda_{\star}$ (which has exponential dependence on the dimension) can be avoided in the analysis of \cite{raginsky2017nonconvex}, although in their case this would not have improved the final result due to the contribution of the computational excess risk.

In this paper we consider the statistical excess risk, while we forego computational aspects of concrete algorithms.
In particular, we focus on the \emph{distribution-dependent} analysis of \emph{statistical excess risk} (or, simply, excess risk).
In the following we show upper bounds on the statistical excess risk that can be much smaller than~\eqref{eq:raginsky_excess} depending on the interplay between the data-generating distribution and the loss function.

The notion of excess risk considered in this paper is defined with respect to the \emph{regularized} minimizer of the risk
\begin{equation}
  \label{eq:regularized_risk}
  \bwstarl \in \argmin_{\bw \in \hypsp}\cbr{R(\bw) + \lambda \|\bw\|^2}~.
\end{equation}
Note that this is not a limitation because we can always recover the regularizer-free analysis by looking at the asymptotic behavior $\lambda \rightarrow 0$.
In the following, we assume that both risk and empirical risk are bounded and twice differentiable, and the Hessian matrix of the risk is locally-Lipschitz (in a sense precisely defined later on).
Therefore, the objective function of~\eqref{eq:regularized_risk} (as well as the one of~\eqref{eq:erm}) can have more than one minimum. However we assume that local minima are \emph{isolated}, meaning that a sufficiently small neighborhood of $\bwstarl$ contains a unique minimum.
One compelling example is a large family of non-convex \emph{strict-saddle} \ac{ERM} problems~\citep{ge2015escaping,gonen2017fast}, such as matrix completion, tensor decomposition, PCA, ICA, and others.
Another example of such is an empirical risk of a \emph{ReLU neural network} with weight decay, or L2 regularization,~\cite[Theorem 1]{milne2018piecewise} where minima resulting in a sufficiently small empirical risk are locally strongly-convex.
Even though the theorem holds for empirical measures, we suspect that it could be extended to the population risk through the uniform convergence argument.
\paragraph{Localized excess risk.}
Before delving into the global analysis of the excess risk, we look at the \emph{local} approximation properties of the Gibbs-ERM principle, which will also be instrumental in the forthcoming global analysis.
We begin by looking at the \emph{localized} excess risk with respect to a fixed  minimizer $\bwstarl$.
Specifically, we consider the excess risk of a parameter $\bw$ generated by Gibbs-ERM within a certain neighborhood around $\bwstarl$. This is defined as
\[
  \excess(\bwstarl) = \E_S\br{\E_{\bw \sim \ps}\br{R(\bw) - R(\bwstarl) \mid \bw \in \sEstar\pr{r}}} \,,
\]
where conditioning is taken on the event that $\bw$ lies in the ellipsoid $\sEstar\pr{r}$ of radius $r$ centered at the minimizer $\bwstarl$ and aligned with the curvature of the risk at that minimum.
We prove (Theorem~\ref{thm:gibbs_local_excess}) that
the local excess risk behaves as%
\footnote{Throughout this paper, we use $f \leqC g$ to say that there exists a universal constant $C > 0$ such that $f \leq C g$ holds uniformly over all arguments.}
\begin{equation}
  \label{eq:gibbs_local_excess}
  \excess(\bwstarl)
  \leqC
  \frac{1}{\gamma} \, \tr\pr{\bHstar \bHstarlinv}
  + \ve(r)
  + \sqrt{\gamma \ve(r)}
  + \frac{\gamma}{\sqrt{m}}
\end{equation}
where $\bHstar$ is the Hessian matrix $\nabla^2 R(\bwstarl)$, $\bHstarl = \bHstar + 2 \lambda \bI$, and $\ve(r)$ is a local approximation error that vanishes as $r\to 0$
(defined precisely in Section~\ref{sec:smooth_loss}).
The trace term in~\eqref{eq:gibbs_local_excess}, a distribution-dependent quantity known as the \emph{effective dimension} of minimizer $\bwstarl$, can be also expressed as $\lambda_1 / (\lambda_1 + \lambda) + \cdots + \lambda_d / (\lambda_d + \lambda)$ where $\lambda_1, \ldots, \lambda_d$ are the eigenvalues of $\bHstar$. This can be viewed as a ``soft'' version of the rank of $\bHstar$.
Note that $\tr\pr{\bHstar \bHstarlinv} \leq d$ always, and $\tr\pr{\bHstar \bHstarlinv} \ll d$ whenever the spectrum of the Hessian matrix is light-tailed.
This notion of effective dimension occasionally appears in the analysis of ridge regression~\citep{audibert2011robust,neu2018iterate}.

Next, to get a sense of the strength of the bound and as a sanity check, one may look at limiting cases
with respect to parameters $\lambda$ and $\gamma$.
When $\lambda \to 0$, corresponding to the unregularized Gibbs-ERM principle, our bound becomes 
\[
  \excess(\bwstar) \leqC \frac{1}{\gamma} \, \rank(\bHstar) + \text{poly}(m, \gamma, r, \lmin(\bHstar))
\]
where $\lmin(\bHstar)$ denotes the smallest non-zero eigenvalue of $\bHstar$.
Assuming the radius is set to $r = \gamma^{-\frac{1+p}{3}}$ for some $p > 0$, the polynomial term in the right-hand side of the above bound vanishes as $\gamma,m\to\infty$, even for singular $\bHstar$.
On the other hand, for $\lambda > 0$ and $\gamma,m\to\infty$
the right-hand side of~\eqref{eq:gibbs_local_excess} tends to zero, and the bound backs up the intuition that the Gibbs-ERM principle should exactly recover the \ac{ERM} solution. This observation also serves as a sanity check that the bound is reasonably tight, at least with respect to $\gamma$.

Finally, in cases when the Hessian matrix of the risk is constant, for instance in \ac{RLS} problems, $\ve(r) = 0$ and our bound specializes to
\begin{equation*}
  \excess(\bwstarl)
  \leqC
  \frac{1}{\gamma} \, \tr\pr{\bHstar \bHstarlinv}
  + \frac{\gamma}{\sqrt{m}}~.
\end{equation*}
When $\gamma$ is tuned optimally the above bound becomes $\excess(\bwstarl) \leqC m^{-\frac{1}{4}}\sqrt{\tr\pr{\bHstar \bHstarlinv}}$.
Note that, for the square loss, the best known dependence on $m$ that can be achieved is $m^{-\frac{1}{2}}$. The worse exponent in our bound is the price we pay for the generality of our approach.
Although our results are never worse in terms of the dimension, the dependence on the sample size in our bounds is worse than in those of~\cite[(3.27)]{raginsky2017nonconvex}.
This is because we are stating bounds in terms of the distribution-dependent effective dimension. We can obtain the same rate as~\cite{raginsky2017nonconvex} by expressing the effective dimension in terms of the empirical risk.

\paragraph{Global excess risk.}
Next, we consider a \emph{global} notion of excess risk,
\[
  \excess(\pi) = \E_S\br{\E_{\bw \sim \ps}\br{R(\bw)}} - \E_{I \sim \pi} \br{R(\bwstarlI)}\,.
\]
Here, in the second term $\pi$ is a distribution over the countable set $\sI$ of all minima (recall that minima are isolated). 
In this setting, all minima, Hessian matrices, approximation errors, and ellipsoids gain a corresponding subscript $i \in \sI$: $\bwstarli$, $\bHstarli$, $\ve_i(r)$, and $\sEstar_i(r)$.

We first focus on the finite-temperature distribution over minima of the regularized risk
\begin{equation}
  \label{eq:pi_gamma}
  \pigamma(i) = \frac{\Pgamma(\sEstar_i(r))}{\sum_{j \in \sI} \Pgamma(\sEstar_j(r)) }\,, \qquad i \in \sI\,,
\end{equation}
where probabilities are taken with respect to the population Gibbs density $p_{\gamma}(\bw) \propto e^{-\gamma \Rl(\bw)}$.
For this distribution we prove that
\begin{equation}
  \label{eq:gibbs_global_excess}
\excess(\pigamma)
  \leqC
  \frac{1}{\gamma} \, \E\br{\tr\pr{\bHstarI \bHstarlIinv}}
  + \frac{\gamma}{\sqrt{m}}
  + \E[\ve_I(r)] + \sqrt{\gamma \E[\ve_I(r)]} + \Pgamma(\mincomp(r))\,,
\end{equation}
for any $r < r_0$ where the radius $r_0$ is chosen such that all ellipsoids in the set $\big\{\sEstar_i(r) \,:\, i \in \sI\big\}$ are disjoint,
$\mincomp(r)$ is the \emph{complement} of the union of the ellipsoids in this set (i.e., the volume outside of minima), and the expectation is taken with respect to $I \sim \pigamma$.

Note that there is a trade-off in~\eqref{eq:gibbs_global_excess} between the first term, which is essentially a bound on the expected excess risk in the neighborhood of a minimum drawn according to $\pigamma$, and the last term, which is the probability of sampling outside of the neighborhood of any minimum.
This means that we can obtain an \emph{oracle inequality} by choosing $r \in [0, r_0]$ such that it minimizes the right-hand side of~\eqref{eq:gibbs_global_excess}.

Now we focus on the probability of the complement, which behaves as
\begin{equation}
  \label{eq:intro_prob_complement}
  \Pgamma(\mincomp(r)) \leq 1 - \pr{1 - d e^{-r^2 \gamma \alpha_{d/2}}} \sum_{i \in \sI} e^{-\frac{1}{3} \gamma \ve_i(r)}\,,
\end{equation}
where $\alpha_{d/2}$ depends only on $d$.
So, as long as $r^2 \gamma$ is increasing and $\gamma \ve_i(r)$ is non-increasing, the probability of generating a solution outside of the minima decreases.
For example, when $r = \gamma^{\frac{p-1}{2}}$ for $p \in (0, 1/3]$ (as discussed in Section~\ref{sec:global}) the right-hand side of~\eqref{eq:intro_prob_complement} vanishes as $\gamma \to \infty$.

\paragraph{Asymptotic pseudo excess risk.}
It is also natural to ask what happens in the zero-temperature regime $\gamma \to \infty$, when the Gibbs-ERM principle reduces to a rule for selecting empirical risk minimizers.
We can study this by observing that~\eqref{eq:gibbs_global_excess} vanishes when the radius is set to $r = \gamma^{\frac{p-1}{2}}$ ---as we previously discussed--- and $\gamma$ is set to $m^{\frac{1}{4}}$,
which is a meaningful result. Indeed, whenever $m = \infty$, then $\gamma = \infty$ and the risk of Gibbs-ERM should not differ from the risk of a minimum drawn from the limiting distribution $\piinfty = \lim_{\gamma \rightarrow \infty} \pigamma$.
Interestingly, the distribution $\piinfty$ has the following analytic form (this is shown in Lemma~\ref{lem:erm_prob} assuming the tuning $r = \gamma^{\frac{p-1}{2}}$):
\begin{equation}
\label{eq:pi_infty}
\piinfty(i) = \frac{1}{\sum_{j \in \sI\glob} \sqrt{\frac{\det(\bHstar_{\lambda, i})}{\det(\bHstar_{\lambda, j})} }}\,, 
	\qquad i \in \sI\glob
\end{equation}
where $\sI\glob$ is a countable set enumerating global minima of the regularized risk.
Hence, the probability of a minimum $\bwstarli$ is proportional to the reciprocal of the normalized volume of the ellipsoid defined by the eigenvalues of the Hessian at that minimum. In particular, this implies that the probability of choosing a global minimum with larger volume is higher.
Note that all suboptimal minima have zero probability under $\piinfty$.
At the same time, in this asymptotic regime it is also rather clear that Gibbs-ERM generates models outside of the neighborhoods of the minima with zero probability.
These two observations show how to strike a middle ground between the nonasymptotic bound of~\eqref{eq:gibbs_global_excess} and the asymptotic distribution~\eqref{eq:pi_infty}. This is captured by the global asymptotic \emph{pseudo} excess risk
\[
  \excess^{\infty}_r = \E_{I \sim \piinfty}\br{\E_S\br{\E_{\bw \sim \ps}\br{R(\bw) \mid \bw \in \sEstar_I(r)}} - R(\bwstarlI)}\,, \qquad r > 0\,,
\]
which bounds the localized excess risk at finite temperature $\gamma > 0$ when minima are drawn from the global limiting distribution $\piinfty$.
For any $r \geq 0$ we have
\begin{equation}
  \label{eq:gibbs_global_pseudo_excess}
    \excess^{\infty}_r \leqC \frac{1}{\gamma} \, \E\br{\tr\pr{\bHstarI \bHstarlIinv}}
    + \E \br{\ve_I(r)}
    + \sqrt{\gamma \E \br{\ve_I(r)}}
  + \frac{\gamma}{\sqrt{m}}~.
\end{equation}
Observe that whereas the local excess risk~\eqref{eq:gibbs_local_excess} is essentially controlled by the soft rank of the minimum, the bound~\eqref{eq:gibbs_global_pseudo_excess} implies that globally this is not necessarily the case, since low-rank minima have smaller probability under distribution $\piinfty$.

\subsection{Additional Related Work}
Generalization bounds for the Gibbs-ERM principle has been extensively studied in a number of works over the past years.
One prolific thread of research in this direction is the PAC-Bayesian analysis starting from seminal works of~\cite{mcallester1998some,langford2003pac} until the latest developments, for instance~\citep{germain2015risk,grunwald2017tight}.
PAC-Bayesian analysis follows the uniform convergence argument (over the class of densities), where, as was pointed out earlier Gibbs density is the one minimizing the bound on the risk (Information Risk Minimization).
In this paper we focus on the excess risk bounds rather than generalization bounds, that manifest distribution-dependent properties of a potentially non-convex risk.
PAC-Bayesian excess risk bounds have also been studied in a number of contexts~\citep{alquier2016properties,audibert2011robust,grunwald2017tight}, however, these typically assume convex risk (e.g.\ least-squares), or focus on the properties of the hypothesis class (sometimes distribution-dependent) rather than those of the objective function~\citep{grunwald2017tight}.
Distribution-dependent arguments have been exploited to develop sharper generalization bounds~\citep{lever2013tighter}, and data-dependent PAC-Bayesian bounds were shown to be numerically non-vacuous as was shown by~\cite{dziugaite2018data}.

A number of works have also analyzed generalization and approximation properties of Gibbs-ERM from the algorithmic point of view.
A heuristic approach to analysis of \ac{SGLD} algorithm was given by \cite{welling2011bayesian,mandt2016variational} while recent works have also argued that its generalization ability is controlled by the ``width'' (or the notion of pseudo-rank) at the minimum of the empirical risk~\citep{keskar2017large,chaudhari2018stochastic,liang2019fisher}, which is reminiscent of the effective dimension studied in this paper.
\cite{mou2018generalization} developed generalization bounds for \ac{SGLD} from PAC-Bayesian and algorithmic stability point of view~\citep{BousquetE02}.
Apart from the excess risk bounds,~\cite{raginsky2017nonconvex} also showed Gibbs-ERM specific generalization bounds through the algorithmic stability framework.
In this paper we also analyze generalization and stability of Gibbs-ERM principle, however we present a simpler proof technique similar in spirit to~\cite{xu2017information}.
Finally,~\cite{sheth2017excess} analyzed a slightly different notion of excess risk bounds for the variational inference assuming the use of latent Gaussian models (such as generalized linear models and Gaussian processes).

\section{Preliminaries}
Throughout this paper, we use $f \leqC g$ to indicate that there exists a universal constant $C > 0$ such that $f \leq C g$ holds uniformly over all arguments.
Let $\sB_r(\bz) \subset \R^d$ be the ball of center $\bz$ and radius $r > 0$ and let $\sB_r = \sB_r(\bzero)$.
Given a positive definite $d \times d$ matrix $\bM$, define $\|\bx\|_{\bM}^2 = \bx\tp \bM \bx$ for $\bx \in \reals^d$.
Then, for any positive semi-definite $d \times d$ matrix $\bA$ and $r > 0$ the corresponding ellipsoid centered at $\bx_0\in\reals^d$ is defined as $\sE(\bx_0, \bA, r) \equiv \cbr{\bx \in \reals^d ~:~ \|\bx_0 - \bx\|_{\bA} \leq r}$.

If $p$ and $q$ are densities that are absolutely continuous with respect to a measure $\mu$ over $\hypsp$, the Kullback-Liebler (KL) divergence between $p$ and $q$ is defined as
\begin{align*}
  \KL(p, q) &= \E_{\bw \sim p}\Big[\ln\big(p(\bw)\big) - \ln\big(q(\bw)\big)\Big]~.
\end{align*}

\section{Sketch of the Analysis}
\label{sec:analysis_idea}
In this section we briefly explain the main arguments at the basis of our analysis.
We start from the analysis of the localized excess risk, which is decomposed into the generalization error (i.e., the difference between the risk and empirical risk) and the gap between the empirical risk and the risk of the minimizer $\bwstarl$,
\begin{align}
\label{eq:gibbs_gen_error_decomp}
\E\br{R(\bw) - R(\bwstarl) \mid \bw \in \sEstar(r)}
  &= \E\br{R(\bw) - \Rhat_S(\bw) \mid \bw \in \sEstar(r)}\\
  \label{eq:gibbs_gen_error_decomp2}
  &+ \E\br{\Rhat_S(\bw) - R(\bwstarl) \mid \bw \in \sEstar(r)}
\end{align}
where the expectation is taken with respect to the empirical Gibbs density~\eqref{eq:gibbs}.
For brevity we omit the indices corresponding to the minima, as in the localized setting we consider a local minimum at the time.
The generalization error of Gibbs-ERM is captured by Theorem~\ref{thm:gibbs_gen_error} below (whose proof can be found in Section~\ref{sec:stability}).
\begin{theorem}[Generalization bound]
  \label{thm:gibbs_gen_error}
Consider any loss function $f : \hypsp \times \sZ \to \reals$ that is $\sigma$-sub-Gaussian in the first argument with respect to the Gibbs density
\begin{equation}
  \label{eq:gibbs_finite_sums}
  \ps(\bw) \propto e^{-\frac{\gamma}{m} \sum_{i=1}^m f(\bw, z_i)} \qquad \gamma > 0
\end{equation}
conditioned on a measurable $A \subseteq \hypsp$. Then the generalization error of Gibbs-ERM satisfies
\[
  \E_S\br{ \E_{\bw \sim \ps}\br{R(\bw) - \Rhat_S(\bw) \,\Big|\, \bw \in A} } \leq \frac{4 \sigma^2 \gamma}{m}~.
\]
\end{theorem}
Assuming that $f$ is everywhere bounded by $M$, Hoeffding Lemma implies the bound $\frac{M^2 \gamma}{2 m}$.
Under the same boundedness assumption, a bound with similar rates was also shown by~\cite{xu2017information} within the mutual information framework.
Our proof works by showing that the Gibbs density is on-average replace-one stable in the sense of~\cite[Section 13]{shalev2014understanding} (we include it here for completeness).

The second quantity~\eqref{eq:gibbs_gen_error_decomp2} is less straightforward to control in a distribution-dependent setting.
\cite{raginsky2017nonconvex} give a global upper bound
\[
  \Rhat_S(\bw) - R(\bwstarl) = \sctilO\pr{\frac{d}{\gamma}}
\]
by further decomposing~\eqref{eq:gibbs_gen_error_decomp2} as follows
\[
  \E_{\bw \sim \ps}[\Rhat_S(\bw)] - \min_{\bu \in \hypsp} \Rhat_S(\bu) + \min_{\bu \in \hypsp} \Rhat_S(\bu) - \E_S[\Rhat(\bwstarl)]~.
\]
The second term is bounded trivially, while the analysis of the first term follows the so called ``almost ERM'' argument. In other words, understanding how ``close'' the solutions generated by Gibbs-ERM are to the solutions of \ac{ERM}.

In our distribution-dependent setting we follow a different route: consider the case of \ac{RLS}, where the empirical risk is an average of square-regularized losses.
A rather straightforward argument (based on Gaussian integration) gives the following \emph{identity} in an ``almost ERM'' style:
\begin{align*}
  \E_{\bw \sim \ps}\br{\Rhat_S(\bw)} - \min_{\bu \in \hypsp} \Rhat_S(\bu)
  =
  \frac{1}{\gamma} \, \tr\pr{\nabla^2 \Rhat_S(\bhatwl) \pr{\nabla^2 \Rhat_S(\bhatwl) + \lambda \bI}^{-1} }
\end{align*}
where $\bhatwl$ is a minimizer of the \ac{RLS} problem.
Observe that in the above identity we obtain an empirical counterpart of the effective dimension introduced in~\eqref{eq:gibbs_local_excess}.
Since our goal is a distribution-dependent result, one possibility is to consider the concentration of Hessian eigenvalues. However, we follow a more direct approach. As we said earlier, the ``almost ERM'' analysis of \ac{RLS} is relatively easy using Gaussian integration (since $\ps$ is a Gaussian density).
Since we deal instead with general smooth densities, our idea is to quantify the gap between the density at hand and the Gaussian density.
This is nicely handled by the \emph{transportation lemma}~\citep[Lemma 4.18]{boucheron2013concentration}, characterizing the difference between the expectations of different densities in terms of the KL-divergence between them.
As a comparison we choose the Gaussian density
$
  \qgamma \propto \exp\big(-\frac{\gamma}{2} \|\bw - \bwstarl\|_{\bHstarl}^2\big)
$,
which only depends only on distribution-related quantities.
An ``almost ERM'' style analysis applied to $\qgamma$ will give us exactly a distribution-dependent effective dimension.
So, the only issue is the actual gap between densities. By expanding the KL-divergence we observe that the critical terms are
\begin{align}
  &- \gamma \E_S\br{\E_{\ps}\br{\Rhatl(\bw) \mid \bw \in \sEstar\pr{r}}}
  + \frac{\gamma}{2} \E_S\br{\E_{\ps}\br{ \|\bw - \bwstarl\|_{\bHstarl}^2 \mid \bw \in \sEstar\pr{r} }} \nonumber\\
  \leqC \ &-\gamma \E_S\br{\E_{\ps}\br{\Rl(\bw) \mid \bw \in \sEstar\pr{r}}} + \frac{\gamma^2}{m}
  + \frac{\gamma}{2} \E_S\br{ \E_{\ps}\br{ \|\bw - \bwstarl\|_{\bHstarl}^2 \mid \bw \in \sEstar\pr{r} }} \label{eq:analysis_explained_1}\\
  \leqC \ &\gamma \ve(r) + \frac{\gamma^2}{m}~. \label{eq:analysis_explained_2}
\end{align}
To obtain~\eqref{eq:analysis_explained_1}, instead of using concentration, we resort to the generalization bound of Theorem~\ref{thm:gibbs_gen_error}, whereas~\eqref{eq:analysis_explained_2} is obtained by Taylor expansion of the regularized risk around its minimizer.
This is formally shown in Lemma~\ref{lem:truncated_EER_to_ERM}, while the decomposition~\eqref{eq:gibbs_gen_error_decomp} is bounded in Theorem~\ref{thm:gibbs_local_excess}.
Hence, the gap is quantified by the approximation error at the radius $r$ plus a sample-dependent term due to the use of empirical Gibbs density.
These terms appear in excess risk bounds~\eqref{eq:gibbs_local_excess},~\eqref{eq:gibbs_global_excess}, and~\eqref{eq:gibbs_global_pseudo_excess}.
\subsection{Global Analysis}
Starting from the conditional local excess risk in the form given by the left-hand side of~\eqref{eq:gibbs_gen_error_decomp}, we analyze a notion of global risk by bounding the probability of individual ellipsoids (i.e., neighborhoods of minima) and that of the complement of their union. Since for $\gamma \to \infty$ the probability of a complement approaches zero (as discussed in Remark~\ref{rem:tuning_r}), in order to obtain an asymptotic bound it is enough to focus on the relative probability of ellipsoids. In Lemma~\ref{lem:erm_prob} we derive upper and lower bounds on this probability via Laplace approximation (Lemma~\ref{lem:ellipse_prob}), and then analyze their limit for $\gamma \to \infty$. Combining the local excess risk bound and the bound on the asymptotic relative probability of ellipsoids allows us to control the asymptotic global pseudo excess risk~\eqref{eq:gibbs_global_pseudo_excess} ---see Corollary~\ref{cor:global_excess_risk_asymptotic}.

Finally, using a nonasymptotic bound on the probability of a complement ---see the proof of Theorem~\ref{thm:global_excess_risk}--- we can apply the law of total expectation to get also a nonasymptotic bound on the global excess risk.

\section{Main results}
\label{sec:smooth_loss}
\subsection{Local analysis}
We first turn our attention to the local analysis considering a fixed minimizer\footnote{We will drop subscript indexing of minima in this section.}
\begin{equation}
  \bwstarl \in \argmin\cbr{R(\bw) + \lambda \|\bw\|^2}~.
\end{equation}
Specifically, we prove that the risk of Gibbs-ERM in a neighborhood of $\bwstarl$ is controlled by the \emph{local} effective dimension
$
  \tr\pr{\bHstar \bHstarlinv}
$,
defined in terms of the Hessian $\bHstar = \nabla^2 R(\bwstarl)$ of the risk of the minimizer, where $\bHstarl = \bHstar + 2 \lambda \bI$.
We require that Hessians do not change ``too quickly'' by assuming that Hessian of the risk is Lipschitz in an ellipsoid $\sEstar(r) = \sE(\bwstarl, \bHstarl, r)$ centered at the minimizer and aligned with the local curvature. Formally the \emph{local Lipschitzness} of the Hessian is defined as follows.
\begin{definition}[Locally-Lipschitz Hessian]
The Hessian $\nabla^2 R$ is locally Lipschitz around a minimizer $\bwstarl$ if there exists a function $\Lip : \reals_+ \to \reals_+$ such that
\begin{equation}
  \label{eq:lipschitz_hessian}
  \|\nabla^2 R(\bwstarl) - \nabla^2 R(\bw)\|_2 \leq \Lip(r) \|\bwstarl - \bw\| \qquad \text{for all} \quad \bw \in \sEstar(r)~.
\end{equation}
\end{definition}
Note that local Lipschitzness of the risk Hessian implies local Lipschitzness of the regularized risk Hessian for the same function $\Lip$.

Local Lipschitzness of Hessians plays an important role in bounding the gap between the Gibbs density and the Gaussian density, as discussed in Section~\ref{sec:analysis_idea}.
In particular, the approximation error introduced by taking a Taylor expansion of the regularized risk up to the third term is 
\begin{equation}
  \label{eq:approx}
  \ve(r) = \Lip\pr{r} \pr{\frac{r}{\sqrt{\lmin + \lambda}}}^3 \qquad r \geq 0
\end{equation}
where $\lmin$ is the smallest non-zero eigenvalue of $\bHstar$.
Observe that $\lim_{r \rightarrow \infty}\ve(r) = 0$ because $\Lip(r)$ converges to $\|\nabla^2 R(\bwstarl)\|_2$, and that $\ve(r) = 0$ for any constant Hessian matrices (e.g., in the case of \ac{RLS}).
Lemma~\ref{lem:truncated_EER_to_ERM} below here establishes a result needed to prove our bound on the local excess risk.
\begin{lemma}
\label{lem:truncated_EER_to_ERM}
For any minimizer $\bwstarl$ of the regularized risk we have
\begin{align*}
  \E_S\br{ \E_{\bw \sim \ps}\br{\Rhat_S(\bw) \,\Big|\, \bw \in \sEstar\pr{r}} } - R(\bwstarl)
  \leq
    \frac{1}{\gamma} \, \tr\pr{\bHstar \bHstarlinv} + \frac{\ve(r)}{6}
  + \frac{M}{2} \sqrt{\frac{\gamma \ve(r)}{3} + \frac{M^2 \gamma^2}{2 m}}~.
\end{align*}
\end{lemma}
(The proof of Lemma~\ref{lem:truncated_EER_to_ERM}, along with that of all remaining statements in this section, can be found in Section~\ref{sec:proofs_smooth}.)
This lemma, combined with the bound on the generalization error (Theorem~\ref{thm:gibbs_gen_error}), gives the desired result.
\begin{theorem}[Localized Excess Risk Bound]
  \label{thm:gibbs_local_excess}
  Assume the same as in Lemma~\ref{lem:truncated_EER_to_ERM}.
  Then,
  \begin{equation*}
    \excess(\bwstarl)
    \leq
    \frac{1}{\gamma} \, \tr\pr{\bHstar \bHstarlinv} + \frac{\ve(r)}{6}
    + \frac{M}{2} \sqrt{\frac{\gamma \ve(r)}{3} + \frac{M^2 \gamma^2}{2 m}}
    + \frac{M^2 \gamma}{2 m}~.
  \end{equation*}
\end{theorem}

\subsection{Global analysis}
\label{sec:global}
We now turn our attention to the global analysis of the excess risk.
Since we deal with a countable set of local minima (indexed by $\sI$), we add a subscript to all minima-dependent quantities, such as $\bwstarli$, $\bHstarli$, $\sEstar_i(r)$. In particular, the approximation error is now defined as
\begin{equation}
  \label{eq:approx_global}
  \ve_i(r) = \Lipi\pr{r} \pr{\frac{r}{\sqrt{\lmini + \lambda}}}^3 \qquad r > 0, i \in \sI
\end{equation}
where $\Lipi$ is the local Lipschitz constant with respect to the minimum $\bwstarli$, and $\lmini$ is the smallest non-zero eigenvalue of the Hessian matrix $\bHstarli$.

Next, we introduce an important assumption on the geometry of the regularized risk around its minimizers.
\begin{assumption}
  \label{asm:global}
All local minima $\bwstarl \in \argmin_{\bw \in \hypsp}\Rl(\bw)$ satisfy
  $\nabla \Rl(\bwstarl) = \bzero$ and are such that $\nabla^2 \Rl(\bwstarl)$ is positive definite.
  In other words all local minima are \emph{isolated}.
\end{assumption}
The above assumption implies that there exists a number $r_0 > 0$ such that
\[
  r_0 = \max\cbr{ r > 0 ~:~ \bigcap_{i \in \sI} \sEstar_i(r) \equiv \varnothing }~.
\]
In other words, ellipsoids centered at minimizers and aligned with the local curvature of $\Rl$ are non-overlapping.
In addition to the set $\sI$, indexing minima of the regularized risk, let $\sI\glob \subseteq \sI$ index the global minima and denote its complement by $\sI\subopt \equiv \sI \setminus \sI\glob$.
Finally, introduce the complement of the ellipsoids centered at the minima (later called, with some abuse of terminology, \emph{complement of the minima}),
\[
  \mincomp(r) \equiv \hypsp \setminus \bigcup_{i \in \sI} \sEstar_i(r) \qquad r \leq r_0~.
\]
The first result in this section concerns the distribution of local minima.
In particular, we give an upper bound on the relative probability $\pigamma(i)$ of the $i$-th minimum, and then derive the analytic form of the asymptotic distribution $\piinfty$.
\begin{lemma}[Distribution of Minima]
  \label{lem:erm_prob}
  For all $r > 0$,
  \begin{align*}
    \pigamma(i)
    \leq
    \frac{e^{\frac{\gamma}{3} \max_{k \in \sI}\ve_k(r)}}
    {\sum_{j \in \sI} e^{\gamma \pr{ \Rl(\bwstarli) - \Rl(\bwstarlj)}}
    \sqrt{\frac{\det(\bHstarli)}{\det(\bHstarlj)} }
    }
    \qquad i \in \sI~.
  \end{align*}
  Moreover, assuming without loss of generality that $\Rl(\bwstarli) = 0$ for all $i \in \sI\glob$, and setting $r = \gamma^{-p}$ for $p > 0$, we have
  \[
    \lim_{\gamma \rightarrow \infty}\pigamma(i)
    = \begin{cases}
    \frac{1}{\sum_{j \in \sI\glob} \sqrt{\frac{\det(\bHstar_{\lambda, i})}{\det(\bHstar_{\lambda, j})} }} & i \in \sI\glob
  \\
    0 & i \in \sI\subopt~.
  \end{cases}
  \]
\end{lemma}
We are now ready to state the main result of this section.
\begin{theorem}[Global Excess Risk Bound]
  \label{thm:global_excess_risk}
  Assume the same as in Lemma~\ref{lem:truncated_EER_to_ERM}.
  Then for any $r \in [0, r_0]$ the global excess risk satisfies
  \begin{equation*}
  \excess(\pigamma)
  \leqC
  \frac{1}{\gamma} \, \E\br{\tr\pr{\bHstarI \bHstarlIinv}}
  + \frac{\gamma}{\sqrt{m}}
  + \E[\ve_I(r)] + \sqrt{\gamma \E[\ve_I(r)]} + \Pgamma(\mincomp(r))
  \end{equation*}
  where the expectation is taken with respect to $I \sim \pigamma$
  and the probability of the complement of the minima is bounded as
  \begin{equation}
  \label{eq:lb_P_complement}
  \Pgamma(\mincomp(r))
  \leq
  1 - \pr{1 - d e^{-r^2 \gamma \alpha_{d/2}}} \sum_{i \in \sI} e^{-\frac{1}{3} \gamma \ve_i(r)}
  \end{equation}
  with
    \begin{equation}
    \alpha_{d/2} =
    \begin{cases}
      1 & d = 1\\
      \Gamma\pr{1+\frac{d}{2}}^{-\frac{2}{d}} & \text{otherwise}~.
    \end{cases}
  \end{equation}  
\end{theorem}
\begin{remark}
  \label{rem:tuning_r}
  We compute the value of $r$ approximately minimizing the the right-hand side in Theorem~\ref{thm:global_excess_risk}.
  Note that the probability of the complement of the minima decreases if we ensure that $r^2 \gamma$ increases in $\gamma$ and $\gamma \ve_i(r) \propto r^3 \gamma$ is non-increasing.
For instance we may set $r^2 \gamma = \gamma^p$ for $p > 0$
so that $r^3 \gamma = \gamma^{1 + \frac{3}{2} (p-1)}$. Hence we require $1 + \frac{3}{2} (p-1) \leq 0$ which is satisfied for any $p \in (0, 1/3]$.
This implies that when $r = \gamma^{\frac{p-1}{2}}$ and $p \in (0, 1/3]$ the probability of the complement of the minima and the approximation terms $\gamma \ve_i(r), \ve_i(r)$ all vanish as $\gamma \to \infty$.
\end{remark}
Finally, combining the localized excess risk bound in Theorem~\ref{thm:gibbs_local_excess} with Lemma~\ref{lem:erm_prob} allows us to prove the following result about the asymptotic pseudo excess risk.
\begin{cor}
  \label{cor:global_excess_risk_asymptotic}
  Assume the same as in Lemma~\ref{lem:truncated_EER_to_ERM}.
  Then, for any $r > 0$, the global asymptotic pseudo-excess risk satisfies
  \[
    \excess^{\infty}_r
    \leqC
    \frac{1}{\gamma} \, \E\br{\tr\pr{\bHstarI \bHstarlIinv}}
    + \E \br{\ve_I(r)}
    + \sqrt{\gamma \E \br{\ve_I(r)} + \frac{\gamma^2}{m}}
  + \frac{\gamma}{m}
\]
where $I$ is distributed according to
\[
\piinfty(i) =
  \frac{1}{\sum_{j \in \sI\glob} \sqrt{\frac{\det(\bHstar_{\lambda, i})}{\det(\bHstar_{\lambda, j})} }}~.
\]
\end{cor}

\section{Proofs}
\subsection{Common Tools}
We compute the Taylor expansion of $\Rl(\bw)$ for $\bw \in \sE(\bwstarl, \nabla^2 \Rl(\bwstarl), r)$, where $\bwstarl$ is a minimizer of the regularized risk,
\begin{align}
  \Rl(\bw)
  &\geq
    \Rl(\bwstarl) + \frac{1}{2} \|\bw - \bwstarl\|_{\bHstarl}^2 - \frac{\Lip(r)}{6} \, \|\bw - \bwstarl\|^3 \nonumber\\
  &\geq
    \Rl(\bwstarl) + \frac{1}{2} \|\bw - \bwstarl\|_{\bHstarl}^2 - \frac{\Lip(r)}{6} \pr{\frac{r}{\sqrt{\lmin + \lambda}}}^3 \label{eq:taylor_lower_lmin}\\
  &=
    \Rl(\bwstarl) + \frac{1}{2} \|\bw - \bwstarl\|_{\bHstarl}^2 - \frac{1}{6} \ve(r)~ \label{eq:taylor_lower}
\end{align}
where $\ve(r)$ is defined in~\eqref{eq:approx}, and~\eqref{eq:taylor_lower_lmin} follows because $\sqrt{\lmin + \lambda} \|\bw - \bwstarl\| \leq r$ where $\lmin$ is the smallest non-zero eigenvalue of $\nabla^2 R(\bwstarl)$. In a similar way we have the upper expansion
\begin{equation}
  \label{eq:taylor_upper}
  \Rl(\bw)
  \leq
  \Rl(\bwstarl) + \frac{1}{2} \|\bw - \bwstarl\|_{\bHstarl}^2 + \frac{1}{6} \ve(r)~.
\end{equation}
We now introduce a crucial \emph{transportation lemma} which is instrumental in the following proofs.
\begin{lemma}[\protect{\citep[Lemma~4.18]{boucheron2013concentration}}]
  \label{lem:transportation}
  Let $Z$ be a real-valued integrable random variable with distribution $P$ such that
  \[
    \ln \E\br{e^{\alpha (Z - \E[Z])}} \leq \frac{\alpha^2 \sigma^2}{2} \qquad \alpha > 0
  \]
  for some $\sigma > 0$ and let $Z'$ be another random variable with distribution $Q$.
  If $Q$ is absolutely continuous with respect to $P$ and such that $\KL(Q \ || \ P) < \infty$, then
  $
    \E[Z'] - \E[Z] \leq \sqrt{2 \sigma^2 \KL\pr{Q \ || \ P}}
  $.
\end{lemma}
Next, we prove a helpful lemma about the log-ratio of Gibbs integrals.
\begin{lemma}
  \label{lem:ln_Z_Z_bound}
  Let $f_A, f_B : \sX \to \reals$ such that
  \[
    Z_A = \int_{\sB} e^{-\gamma f_A(\bx)} \diff \bx
  \]
  is finite for all $\gamma > 0, \sB \subseteq \sX$
  and let
  \[
    p_A(\bx) = \frac{1}{Z_A} \, e^{-\gamma f_A(\bx)} \quad \gamma > 0, \ \bx \in \sB
  \]
  where $f_B$ is similarly defined.
  Whenever $Z_A > 0$ we have that
  \[
    - \ln\pr{\frac{Z_A}{Z_B}}
    \leq
    \gamma \int_{\sB} p_B(\bx) \pr{f_A(\bx) - f_B(\bx)} \diff \bx~.
  \]
\end{lemma}
\begin{proof}
  Observe that
  \begin{align*}
    \frac{Z_A}{Z_B} &=
                      \frac{\int_{\sB} e^{-\gamma f_A(\bx)} \diff \bx}{\int_{\sB} e^{-\gamma f_B(\bx)} \diff \bx}
                    =
                      \frac{\int_{\sB} e^{-\gamma f_A(\bx)} e^{\gamma \pr{f_B(\bx) - f_B(\bx)}} \diff \bx}{\int_{\sB} e^{-\gamma f_B(\bx)} \diff \bx}
                    =
                      \int_{\sB} p_B(\bx) e^{\gamma \pr{f_B(\bx) - f_A(\bx)}} \diff \bx~.
  \end{align*}
Since $-\ln()$ is a convex function, by Jensen's inequality we obtain the desired result.
\end{proof}
\subsection{Generalization Bound for Gibbs-ERM}
\label{sec:stability}
We start by proving a generalization bound for Gibbs-ERM.
\paragraph{Theorem~\ref{thm:gibbs_gen_error} (restated)}
Consider any loss function $f : \hypsp \times \sZ \to \reals$ that is $\sigma$-sub-Gaussian in the first argument with respect to the Gibbs density
\[
	\ps(\bw) \propto e^{-\frac{\gamma}{m} \sum_{i=1}^m f(\bw, z_i)} \qquad \gamma > 0
\]
conditioned on a measurable $A \subseteq \hypsp$. Then the generalization error of Gibbs-ERM satisfies
\[
  \E_S\br{ \E_{\bw \sim \ps}\br{R(\bw) - \Rhat_S(\bw) \,\Big|\, \bw \in A} } \leq \frac{4 \sigma^2 \gamma}{m}~.
\]
\begin{proof} 
Consider the training examples $S$ drawn i.i.d.\ from $\sD$ and, for $i=1,\ldots,m$, denote by $\Srep = \cbr{z_1, \ldots, z_{i-1}, z, z_{i+1}, \ldots, z_m}$ a replace-one training data, where $z$ is independently drawn from $\sD$.
Throughout the proof, we drop $\gamma$ from the notation for the Gibbs density $\ps$.
Introduce the conditional Gibbs densities
$
  \pscondA(\bw)
% = \frac{\ind{\bw \in A}}{\P_{\psng}(A)} \, \psng(\bw)
$
and
$
  \psrepcondA(\bw)
% = \frac{\ind{\bw \in A}}{\P_{\psrep}(A)} \, \psrep(\bw)
$.
We denote by $\E_{\pscondA}[\cdot]$ and $\E_{\psrepcondA}[\cdot]$ expectations with respect to $\pscondA$ and $\psrepcondA$.
We start by rewriting the expected generalization error as
\begin{align}
  \E_S \E_{\psng}\br{R(\bw) - \Rhat_S(\bw) \mid \bw \in A}
  &= \E_S\br{ \E_{\pscondA}\br{R(\bw) - \Rhat_S(\bw) } }\\
  &= \E_{S,z}\br{ \E_{\pscondA}\br{f(\bw, z)} } - \frac{1}{m} \sum_{i=1}^m \E_{S}\br{ \E_{\pscondA}\br{f(\bw, z_i)} } \label{eq:gibbs_replace_one_stab_1}\\
  &= \frac{1}{m} \sum_{i=1}^m \pr{ \E_{S,z}\br{ \E_{\psrepcondA}\br{f(\bw, z_i)} } - \E_{S}\br{ \E_{\pscondA}\br{f(\bw, z_i)} } } \tag{switch $z$ and $z_i$ in the first term}\\
  &= \frac{1}{m} \sum_{i=1}^m \pr{ \E_{S,z}\br{ \E_{\psrepcondA}[f(\bw, z_i)] - \E_{\pscondA}[f(\bw, z_i)] } }~. \label{eq:gibbs_replace_one_stab}
\end{align}
Now we bound~\eqref{eq:gibbs_replace_one_stab} showing the average replace-one stability of Gibbs distribution.
We use the transportation Lemma~\ref{lem:transportation} with $Q = \pscondA$ and $P = \psrepcondA$ we get that
\begin{equation}
  \label{eq:stab_bound_sqrt_sigma_KL}
  \E_{\psrepcondA}[f(\bw, z_i)] - \E_{\pscondA}[f(\bw, z_i)] \leq \sqrt{2 \sigma^2 \KL\pr{\psrepcondA ~||~ \pscondA}}~.
\end{equation}
Next, we focus on $\KL$-divergence,
\begin{align}
  \KL\pr{\psrepcondA ~||~ \pscondA}
  &=
    \gamma \E_{\psrepcondA}\br{\Rhat_S(\bw) - \Rhat_{\Srep}(\bw)} - \ln\pr{\frac{Z_{\Srep}}{Z_S} \, \frac{\P_{\psrep}(A)}{\P_{\psng}(A)} }\\
  &=
    \gamma \E_{\psrepcondA}\br{\Rhat_S(\bw) - \Rhat_{\Srep}(\bw)} - \ln\pr{\frac{\int_A e^{-\gamma \Rhat_{\Srep}(\bw)} \diff \bw}{\int_A e^{-\gamma \Rhat_S(\bw)} \diff \bw} }\\
  &\leq
    \gamma \E_{\psrepcondA}\br{\Rhat_S(\bw) - \Rhat_{\Srep}(\bw)}
    +
    \gamma \E_{\pscondA}\br{\Rhat_{\Srep}(\bw) - \Rhat_S(\bw)} \tag{by Lemma~\ref{lem:ln_Z_Z_bound}}\\
  &=
    \frac{\gamma}{m} \E_{\psrepcondA}\br{ f(\bw, z_i) - f(\bw, z) } +
    \frac{\gamma}{m} \E_{\pscondA}\br{f(\bw, z) - f(\bw, z_i)}\\
  &=
    \frac{\gamma}{m} \pr{ \E_{\psrepcondA}\br{f(\bw, z_i)} - \E_{\pscondA}\br{f(\bw, z_i)} }\\
  &+\frac{\gamma}{m} \pr{ \E_{\pscondA}\br{f(\bw, z)} - \E_{\psrepcondA}\br{f(\bw, z)} } ~. \label{eq:gibbs_stability_bound_on_kl}
\end{align}
By taking expectation with respect to $S$ and $z$ on both sides, we get that the first term in~\eqref{eq:gibbs_stability_bound_on_kl} can be expressed as
\begin{equation}
  \label{eq:gibbs_stability_expectation_of_kl}
  \E_{S,z}\br{ \E_{\psrepcondA}\br{f(\bw, z_i)} - \E_{\pscondA}\br{f(\bw, z_i)} } = \E_{S,z}\br{\E_{\pscondA}\br{f(\bw, z)} - \E_{\psrepcondA}\br{f(\bw, z)}}
\end{equation}
where we could switch $z_i$ and $z$ on the right-hand side because their are both independently drawn from $\sD$.
Thus, the expectation of~\eqref{eq:stab_bound_sqrt_sigma_KL} with respect to $S$ and $z$ is upper-bounded as
\begin{align}
  \E_{S,z}\br{\E_{\psrepcondA}[f(\bw, z_i)] - \E_{\pscondA}[f(\bw, z_i)]}
  &\leq
    \E_{S,z}\sqrt{2 \sigma^2 \KL\pr{\psrepcondA ~||~ \pscondA}}\\
  &\leq
    \sqrt{2 \sigma^2 \E_{S,z}\br{\KL\pr{\psrepcondA ~||~ \pscondA}}} \tag{Jensen's inequality}\\
  &\leq
    2 \sqrt{\frac{\sigma^2 \gamma}{m}
    \E_{S,z}\br{\E_{\psrepcondA}\br{f(\bw, z_i)} - \E_{\pscondA}\br{f(\bw, z_i)}}  }~.
\end{align}
Solving the above with respect to the term on the left-hand side we get that for any $i=1,\ldots,m$,
\begin{align}
  \E_{S,z}\br{\E_{\psrepcondA}[f(\bw, z_i)] - \E_{\pscondA}[f(\bw, z_i)]} \leq \frac{4 \sigma^2 \gamma}{m}~.
\end{align}
Substituting the above into~\eqref{eq:gibbs_replace_one_stab_1} gives the desired generalization bound.
\end{proof}
\subsection{Localized Excess Risk Bounds}
\label{sec:proofs_smooth}
First, we prove a key lemma about the conditional expectation of quadratic forms.
\begin{lemma}
  \label{lem:conditional_trace_lemma}
  Suppose that $\bx \sim \sN(\bzero, \bM)$.
  Then for the ellipsoid
  \[
    \sE(r) \equiv \cbr{\bx \in \reals^d ~:~ \|\bx\|_{\bM^{-1}} \leq r} \qquad r > 0
  \]
  and for any \ac{PSD} $d \times d$ matrix $\bA$
  we have that
  \begin{align*}
    \E\br{\bx\tp \bA \bx \,\Big|\, \bx \in \sE(r)} &= \frac{F_{d+2}(r^2)}{F_{d}(r^2)} \, \tr\pr{\bA \bM}
                                                 \leq \tr\pr{\bA \bM}~.
  \end{align*}
  where $F_k$ is the CDF of a $\sX^2$-distribution with $k$ degrees of freedom.

  Moreover the above implies that
  \begin{equation}
    \lim_{r \rightarrow 0} \E\br{\bx\tp \bA \bx \,\Big|\, \bx \in \sE(r)} = 0~.
  \end{equation}
\end{lemma}
\begin{proof}
  Observe that
  \begin{align}
    \E\br{\bx\tp \bA \bx \,\Big|\, \bx \in \sE(r)}
    %&= \E\br{\tr\pr{\bx\tp \bA \bx} \mid \bx \in \sE_{r}}\\
    &= \E\br{\tr\pr{\bA \bx \bx\tp} \,\Big|\, \bx \in \sE(r)}\\
    &= \tr\pr{\bA \E\br{\bx \bx\tp \,\Big|\, \bx \in \sE(r)}} \tag{by linearity of trace}\\
    &= \tr\pr{\bA \btilM}
  \end{align}
  where $\btilM$ is the covariance matrix of the Gaussian density $\sN(\bzero, \bM)$ conditioned on $\sE(r)$.
Next, we apply a result about moments of multivariate Gaussian densities under elliptical truncation~\cite[p. 941]{tallis1963elliptical} to get that 
  \begin{align}
    \btilM &= \frac{F_{d+2}(r^2) - F_{d+2}(0)}{F_{d}(r^2) - F_{d}(0)} \, \bM
           = \frac{F_{d+2}(r^2)}{F_{d}(r^2)} \, \bM \label{eq:ratio_of_chi_square_CDF}
  \end{align}
  where $F_d$ is a CDF of a $\sX^2$ distribution.
  This proves the first identity.
    
  The inequality is proven by expanding $\btilM$ further in terms of the Gamma function $\Gamma(\cdot)$ and the incomplete Gamma function $\gamma(\cdot, \cdot)$:
  \begin{align*}
           \btilM &= \frac{\Gamma\pr{\frac{d}{2}}}{\Gamma\pr{1 + \frac{d}{2}}} \, \frac{\gamma\pr{1 + \frac{d}{2}, \frac{r^2}{2}}}{\gamma\pr{\frac{d}{2}, \frac{r^2}{2}}} \, \bM
           = \frac{2}{d} \, \frac{\gamma\pr{1 + \frac{d}{2}, \frac{r^2}{2}}}{\gamma\pr{\frac{d}{2}, \frac{r^2}{2}}} \, \bM
	\preceq \bM~.
  \end{align*}
  Finally, we look at the limit of the ratio in the right-hand side of~\eqref{eq:ratio_of_chi_square_CDF} as $r \rightarrow 0$.
  By L'H\^opital's rule,
 \begin{align*}
    \lim_{r \rightarrow 0} \frac{F_{d+2}(r^2)}{F_{d}(r^2)}
    &=
      \lim_{r \rightarrow 0} \frac{\sX^2_{d+2}(r^2)}{\sX^2_{d}(r^2)}\\
    &=
      \lim_{r \rightarrow 0} \frac{r^d e^{-\frac{r^2}{2}}}{2^{1 + \frac{d}{2}} \Gamma\pr{1 + \frac{d}{2}}}
      \,
      \frac{2^{\frac{d}{2}} \Gamma\pr{\frac{d}{2}}}{r^{d - 2} e^{-\frac{r^2}{2}}}\\
    &=
      \lim_{r \rightarrow 0} \frac{r^d e^{-\frac{r^2}{2}}}{2^{1 + \frac{d}{2}} \Gamma\pr{1 + \frac{d}{2}}}
      \,
      \frac{2^{\frac{d}{2}} \Gamma\pr{\frac{d}{2}}}{r^{d - 2} e^{-\frac{r^2}{2}}}\\
    &=
      \lim_{r \rightarrow 0} \frac{r^2}{d}
    = 0
  \end{align*}
concluding the proof.
\end{proof}
Recall that $\sEstar\pr{r} \equiv \sE(\bwstarl, \bHstarl, r)$ is the ellipsoid of radius $r$ centered at $\bwstarl$.
\paragraph{Lemma~\ref{lem:truncated_EER_to_ERM} (restated)}
For any minimizer $\bwstarl$ of the regularized risk we have
\begin{align*}
  \E_S\br{ \E_{\bw \sim \ps}\br{\Rhat_S(\bw) \mid \bw \in \sEstar\pr{r}} } - R(\bwstarl)
  \leq
    \frac{1}{\gamma} \, \tr\pr{\bHstar \bHstarlinv} + \frac{\ve(r)}{6}
  + \frac{M}{2} \sqrt{\frac{\gamma \ve(r)}{3} + \frac{M^2 \gamma^2}{2 m}}~.
\end{align*}
\begin{proof}
We abbreviate the regularized empirical risk by
$
  \Rhatl(\bw) = \Rhat_S(\bw) + \lambda \|\bw\|^2
$
and recall that the regularized risk is denoted by
$
  \Rl(\bw) = R(\bw) + \lambda \|\bw\|^2
$.
Throughout the proof, we drop $\gamma$ from the notation for the Gibbs densities $\ps$.
Let $\pscond$ be the Gibbs density~\eqref{eq:gibbs} conditioned on the ellipsoid $\sEstar\pr{r}$.
Similarly, let $\qcond$ be the the Gaussian density
\[
  q(\bw) = \frac{1}{Z_q} \, e^{- \frac{\gamma}{2} \|\bw - \bhatw_{\lambda}\|_{\bHstarl}^2} \qquad \bw \in \hypsp
\]
conditioned on $\sEstar\pr{r}$.

We begin by observing that $\Rhat_S$ is trivially $M^2/8$-sub-Gaussian since the loss function is bounded by $M$. Hence, by the transportation Lemma~\ref{lem:transportation},
\begin{align}
\nonumber
  \E_S&\br{ \E_{\psng}\br{\Rhat_S(\bw) \,\Big|\, \bw \in \sEstar\pr{r}} - \E_{q}\br{\Rhat_S(\bw) \,\Big|\, \bw \in \sEstar\pr{r}} }\\
  &= \E_S\br{ \E_{\pscond}\br{\Rhat_S(\bw)} - \E_{\qcond}\br{\Rhat_S(\bw)} } \label{eq:almost_erm_smooth:1}\\
  &\leq \frac{M}{2} \E_S\br{\sqrt{\KL\pr{\pscond \ || \ \qcond}}}
  \leq \frac{M}{2} \sqrt{\E_S\br{\KL\pr{\pscond \ || \ \qcond}}} \label{eq:almost_erm_smooth:transportation_KL}
\end{align}
where the last inequality is obtained by Jensen's inequality.
The KL term can be written as follows
\begin{align}
\nonumber
  \E_S\br{\KL\pr{\pscond \ || \ \qcond}} = &\E_S\E_{\pscond}\br{ \ln\pr{\frac{\pscond(\bw)}{\qcond(\bw)}} }\\
\nonumber
  = &\E_S\E_{\psng}\br{ \ln\pr{\pscond(\bw)} } - \E_S\E_{\psng}\br{ \ln\pr{\qcond(\bw)} }\\
\nonumber
  = &- \gamma \E_S \E_{\psng}\br{ \Rhatl(\bw) \,\Big|\, \bw \in \sEstar\pr{r} } - \E_S\br{\ln(\P_{\psng}(\sEstar\pr{r}) \, Z_{\psng})} \\
\label{eq:almost_erm_smooth:Rhat_lower}
  &+ \frac{\gamma}{2} \E_S \E_{\psng}\br{ \|\bw - \bwstarl\|_{\bHstarl}^2 \,\Big|\, \bw \in \sEstar\pr{r} } + \E_S\br{\ln(\P_{q}(\sEstar\pr{r}) \, Z_{q})}~.
\end{align}
Now we relate the regularized empirical risk~\eqref{eq:almost_erm_smooth:Rhat_lower} to the regularized risk.
By applying Theorem~\ref{thm:gibbs_gen_error} with $A \equiv \sEstar(r)$ we get
\begin{align*}
  \E_S\E_{\psng}\br{\Rl(\bw) - \Rhatl(\bw) \,\Big|\, \bw \in \sEstar\pr{r}}
  &= \E_S\E_{\psng}\br{R(\bw) - \Rhat_S(\bw) \,\Big|\, \bw \in \sEstar\pr{r}}
  \leq \frac{M^2 \gamma}{2 m}~.
\end{align*}
Using this result we can write
\begin{align}
\nonumber
  \E_S\br{\KL\pr{\pscond \ || \ \qcond}} \leq &- \gamma \E_S\E_{\psng}\br{\Rl(\bw) \,\big|\, \bw \in \sEstar\pr{r}} + \frac{M^2 \gamma^2}{2 m}\\
\nonumber
                                         &+ \frac{\gamma}{2} \E_S \E_{\psng}\br{ \|\bw - \bwstarl\|_{\bHstarl}^2 \,\big|\, \bw \in \sEstar\pr{r} }\\
\nonumber
                                         &- \E_S\br{\ln\pr{\frac{\P_{\psng}(\sEstar\pr{r}) \, Z_{\psng}}{\P_{q}(\sEstar\pr{r}) \, Z_{q}}}}\\
  \leq &- \gamma \Rstarl + \frac{\gamma \ve(r)}{6}
         + \frac{M^2 \gamma^2}{2 m} \label{eq:almost_erm_smooth:Rl_taylor_lower}\\
                                         &- \E_S\br{\ln\pr{\frac{\P_{\psng}(\sEstar\pr{r}) \, Z_{\psng}}{\P_{q}(\sEstar\pr{r}) \, Z_{q}}}} \label{eq:almost_erm_smooth:expected_log_ratio}
\end{align}
where~\eqref{eq:almost_erm_smooth:Rl_taylor_lower} is obtained by applying the lower Taylor expansion~\eqref{eq:taylor_lower} to $\bw \in \sEstar\pr{r}$ .
Now we bound the expected log-ratio term in~\eqref{eq:almost_erm_smooth:expected_log_ratio} as
\begin{align}
  - \E_S\br{\ln\pr{\frac{\P_{\psng}(\sEstar\pr{r}) Z_{\psng}}{\P_{q}(\sEstar\pr{r}) Z_{q}}}}
  &=
    - \E_S\br{\ln\pr{\frac{\int_{\sEstar\pr{r}} e^{-\gamma \Rhatl(\bw)} }{\int_{\sEstar\pr{r}} e^{-\frac{\gamma}{2} \|\bw - \bwstarl\|_{\bHstarl}^2} }}}\\
  &\leq
    \gamma \E_S \E_q\br{\Rhatl(\bw) - \frac{1}{2} \|\bw - \bwstarl\|_{\bHstarl}^2 \mid \bw \in \sEstar\pr{r}} \tag{by Lemma~\ref{lem:ln_Z_Z_bound}}\\
  &=
    \gamma \E_q\br{\Rl(\bw) - \frac{1}{2} \|\bw - \bwstarl\|_{\bHstarl}^2 \mid \bw \in \sEstar\pr{r}} \tag{Since $\E_S[\Rhatl(\bw)] = \Rl(\bw)$}\\
  &\leq
    \gamma \Rstarl + \frac{\gamma \ve(r)}{6}~.
\end{align}
where the last inequality is derived from the upper Taylor expansion~\eqref{eq:taylor_upper}.
Substituting the above into~\eqref{eq:almost_erm_smooth:expected_log_ratio} gives
\begin{equation}
\label{eq:bound_KL_term}
  \E_S\br{\KL\pr{\pscond \ || \ \qcond}} \leq \frac{\gamma \ve(r)}{3} + \frac{M^2 \gamma^2}{2 m}~.
\end{equation}
Now we go back to~\eqref{eq:almost_erm_smooth:1} and, using the upper Taylor expansion~\eqref{eq:taylor_upper}, we get
\begin{align}
\nonumber
  \E_S \E_{\qcond}\br{\Rhat_S(\bw)} &= \E_{\qcond}\br{R(\bw)} \tag{since $\E_S[\Rhat_S(\bw)] = R(\bw)$}\\
\nonumber
  &\leq R(\bwstarl)\\
  &+ \nabla R(\bwstarl)\tp \pr{ \E_q \br{\bw - \bwstarl \,\Big|\, \bw \in \sEstar\pr{r}} } \label{eq:almost_erm_smooth:truncated_first_moment}\\
\nonumber
  &+ \E_q\br{ \|\bw - \bwstarl\|_{\bHstar}^2 \,\Big|\, \bw \in \sEstar\pr{r} }
  + \frac{1}{6} \Lip(r) \E_q\br{\|\bw - \bwstarl\|^3 \,\Big|\, \bw \in \sEstar\pr{r} }\\
\nonumber
  &\leq R(\bwstarl)\\
\label{eq:almost_erm_smooth:truncated_second_moment}
  &+ \E_q\br{ \|\bw - \bwstarl\|_{\bHstar}^2 \,\Big|\, \bw \in \sEstar\pr{r} }
  + \frac{\ve(r)}{6}~.
\end{align}
where~\eqref{eq:almost_erm_smooth:truncated_first_moment} vanishes since the first moment of elliptically-truncated Gaussian is zero~\citep{tallis1963elliptical}.
Finally, we bound the first term in~\eqref{eq:almost_erm_smooth:truncated_second_moment} by invoking Lemma~\ref{lem:conditional_trace_lemma}.
By taking $\bM = \gamma \bHstarl$, $\bA = \bHstar$, and $\bx = \bw - \bwstarl$, and using Lemma~\ref{lem:conditional_trace_lemma} we get
\begin{align*}
  \sEstar\pr{r} &\equiv \cbr{\bw \in \reals^d \ : \ \sqrt{\gamma (\bw - \bwstarl)\tp \bHstarl (\bw - \bwstarl)} \leq r }
\end{align*}
and
\begin{align*}
  \E_q\br{ \|\bw - \bwstarl\|_{\bHstar}^2 \,\Big|\, \bw \in \sEstar\pr{r} }
  &\leq
    \frac{1}{\gamma} \, \tr\pr{\bHstar \bHstarlinv}~.
\end{align*}
Now, combining these results with the bound on KL-divergence~\eqref{eq:bound_KL_term}, and
substituting into~\eqref{eq:almost_erm_smooth:transportation_KL}, gives the stated result.
\end{proof}
\paragraph{Theorem~\ref{thm:gibbs_local_excess} (restated)}
Assume the same as in Lemma~\ref{lem:truncated_EER_to_ERM}.
Then,
\begin{equation*}
  \excess(\bwstarl)
  \leq
  \frac{1}{\gamma} \, \tr\pr{\bHstar \bHstarlinv} + \frac{\ve(r)}{6}
  + \frac{M}{2} \sqrt{\frac{\gamma \ve(r)}{3} + \frac{M^2 \gamma^2}{2 m}}
  + \frac{M^2 \gamma}{2 m}~.
\end{equation*}
\begin{proof}
  From the definition of local generalization error,
  \begin{align*}
    \excess(\bwstarl)
    &=
    \E_S\br{ \E_{\bw \sim \ps}\br{R(\bw) \mid \bw \in \sEstar\pr{r}} } - R(\bwstarl)\\
    &= \E_S\br{ \E_{\bw \sim \ps}\br{R(\bw) - \Rhat_S(\bw) \mid \bw \in \sEstar\pr{r}} }\\
    &+ \E_S\br{ \E_{\bw \sim \ps}\br{\Rhat_S(\bw) \mid \bw \in \sEstar\pr{r}} } - R(\bwstarl)\\
    &\leq \frac{M^2 \gamma}{2 m}
    +\frac{1}{\gamma} \, \tr\pr{\bHstar \bHstarlinv} + \frac{\ve(r)}{6}
     + \frac{M}{2} \sqrt{\frac{\gamma \ve(r)}{3} + \frac{M^2 \gamma^2}{2 m}}
  \end{align*}
  where the last inequality is derived from Theorem~\ref{thm:gibbs_gen_error} and Lemma~\ref{lem:truncated_EER_to_ERM}.  
\end{proof}
\subsection{Statements about Probability Mass of Ellipsoids}
Before we prove our bound on the global excess risk, we introduce some necessary technical notions about the \emph{regularized gamma function}, which can be interpreted as the probability of an Euclidean ball of radius $z$ under a Gaussian density with covariance matrix $\bI$.
\begin{theorem}[\protect{\citep[Regularized Gamma Function]{nist2018gamma}}]
Denote the regularized gamma function by
  \[
    P(a, z) = \frac{\Gamma(a) - \Gamma(a,z)}{\Gamma(a)}
  \]
  where $\Gamma(a,z)$ is the upper incomplete Gamma function given by
  \[
    \Gamma(a,z) = \int_z^{\infty} t^{a-1} e^{-t} \diff t~.
  \]
  Then, for all $z \geq 0$ and $a > 0$,
  \begin{equation}
    \label{eq:lb_regularized_gamma}
    \pr{1 - e^{-\alpha_a z}}^a \leq P(a, z)
  \end{equation}
  where
  \begin{equation}
    \label{eq:alpha_a}
    \alpha_a =
    \begin{cases}
      1 & 0 < a < 1\\
      \frac{1}{\Gamma(1+a)^{\frac{1}{a}}} & a > 1~.
    \end{cases}
  \end{equation}  
  with equality in~\eqref{eq:lb_regularized_gamma} only when $a=1$.
\end{theorem}
\begin{prop}[Truncated Gaussian Integrals]
  \label{prop:ball_trunc_gaussian_int}
  For any $\gamma, r > 0$,
  \[
    \int_{\sB(r)} e^{-\frac{\gamma}{2} \|\bu\|^2} \diff \bu = \pr{\frac{2 \pi}{\gamma}}^{\frac{d}{2}} P\pr{\frac{d}{2}, \frac{r^2 \gamma}{2}}
  \]
   where $\sB(r)$ is the $d$-dimensional Euclidean ball.
  In addition, for any $d \times d$ semi-definite matrix $\bA$,
  \[
    \int_{\sE(\bzero, \bA, r)} e^{-\frac{\gamma}{2} \|\bu\|_{\bA}^2} \diff \bu
    =
    \frac{1}{\sqrt{\det(\bA)}} \pr{\frac{2 \pi}{\gamma}}^{\frac{d}{2}} P\pr{\frac{d}{2}, \frac{r^2 \gamma}{2}}~.
  \]
\end{prop}
\begin{proof}
  By the integration of radial functions
  \begin{align*}
    \int_{\sB(r)} e^{-\frac{\gamma}{2} \|\bu\|^2} \diff \bu
    &=
      2 \, \frac{\pi^{\frac{d}{2}}}{\Gamma\pr{\frac{d}{2}}} \int_0^r e^{-\frac{\gamma}{2} x^2} x^{d-1} \diff x \\
    &=
      \pr{\frac{2 \pi}{\gamma}}^{\frac{d}{2}} \, \frac{\Gamma\pr{\frac{d}{2}} - \Gamma\pr{\frac{d}{2}, \frac{r^2 \gamma}{2}}}{\Gamma\pr{\frac{d}{2}}} \\
    &=
      \pr{\frac{2 \pi}{\gamma}}^{\frac{d}{2}} P\pr{\frac{d}{2}, \frac{r^2 \gamma}{2}}~.
  \end{align*}
  In addition we have
  \begin{align*}
    \int_{\sE(\bzero, \bA, r)} e^{-\frac{\gamma}{2} \|\bu\|_{\bA}^2} \diff \bu
    &=
      \int_{\|\bu\|_{\bA} \leq r} e^{-\frac{\gamma}{2} \|\bu\|_{\bA}^2} \diff \bu\\
    &=
      \frac{1}{\sqrt{\det(\bA)}} \int_{\|\bv\| \leq r} e^{-\frac{\gamma}{2} \|\bv\|^2} \diff \bv\\
    &=
      \frac{1}{\sqrt{\det(\bA)}} \pr{\frac{2 \pi}{\gamma}}^{\frac{d}{2}} P\pr{\frac{d}{2}, \frac{r^2 \gamma}{2}}~.
  \end{align*}
  where the third step is obtained through the change of variables $\diff \bA^{\frac{1}{2}} \bu = \diff \bv$.
\end{proof}
Recall that $\sEstar(r) \equiv \sE(\bwstarl, \bHstarl, r)$, where $\bHstarl = \nabla^2 \Rl(\bwstarl)$.
\begin{lemma}[Bounds on the Ellipsoid probability mass.]
  \label{lem:ellipse_prob}
  Let $\bwstarl$ be any minimizer of $\Rl$.
  Then the following results hold for probabilities of ellipsoids under the density $e^{-\gamma \Rl(\bw)} / Z$,
  \begin{align*}
    \P\big(\sEstar(r)\big)
    &\leq
    \frac{1}{Z} \, e^{-\gamma \Rl(\bwstarl) + \frac{\gamma}{6} \ve(r)} \,
    \frac{1}{\sqrt{\det(\bHstarl)}} \pr{\frac{2 \pi}{\gamma}}^{\frac{d}{2}} P\pr{\frac{d}{2}, \frac{r^2 \gamma}{2}}
    \\
    \P\big(\sEstar(r)\big)
    &\geq
    \frac{1}{Z} \, e^{-\gamma \Rl(\bwstarl) - \frac{\gamma}{6} \ve(r)} \,
    \frac{1}{\sqrt{\det(\bHstarl)}} \pr{\frac{2 \pi}{\gamma}}^{\frac{d}{2}} P\pr{\frac{d}{2}, \frac{r^2 \gamma}{2}}
    \\
    \P\big(\sEstar(r)\big) &\geq e^{-\frac{\gamma}{3} \ve(r)} P\pr{\frac{d}{2}, \frac{r^2 \gamma}{2}}~.
  \end{align*}  
\end{lemma}
\begin{proof}
By applying the lower Taylor expansion~\eqref{eq:taylor_lower} in the exponent of the Gibbs density we get
\begin{align}
  \P\big(\sEstar(r)\big) &= \frac{1}{Z} \int_{\sEstar(r)} e^{-\gamma \Rl(\bw)} \diff \bw \nonumber\\
                         &\leq \frac{1}{Z} \, e^{-\gamma \Rl(\bwstarl) + \frac{\gamma}{6} \ve(r)}
                           \int_{\sEstar(r)} e^{- \frac{\gamma}{2} \|\bw - \bwstarl\|_{\bHstarl}^2} \diff \bw \nonumber\\
                         &= \frac{1}{Z} \, e^{-\gamma \Rl(\bwstarl) + \frac{\gamma}{6} \ve(r)} \,
                           \frac{1}{\sqrt{\det(\bHstarl)}} \int_{\sB(r)} e^{- \frac{\gamma}{2} \|\bu\|^2} \diff \bu \label{eq:P_ellipse_bounds_change}\\
                         &= \frac{1}{Z} \, e^{-\gamma \Rl(\bwstarl) + \frac{\gamma}{6} \ve(r)} \,
                           \frac{1}{\sqrt{\det(\bHstarl)}} \pr{\frac{2 \pi}{\gamma}}^{\frac{d}{2}} P\pr{\frac{d}{2}, \frac{r^2 \gamma}{2}}
                           \label{eq:P_ellipse_bounds_gauss_int}
\end{align}
where~\eqref{eq:P_ellipse_bounds_change} is obtained via the change of variables $\bu = \bHstarlhalf (\bw - \bwstarl)$
and~\eqref{eq:P_ellipse_bounds_gauss_int} via Proposition~\ref{prop:ball_trunc_gaussian_int}.
This shows the first result.
The second result follows in a similar way exploiting the upper Taylor expansion~\eqref{eq:taylor_upper},
\begin{align}
  \P\big(\sEstar(r)\big) &\geq \frac{1}{Z} \, e^{-\gamma \Rl(\bwstarl) - \frac{\gamma}{6} \ve(r)}
                           \int_{\sEstar(r)} e^{- \frac{\gamma}{2} \|\bw - \bwstarl\|_{\bHstarl}^2} \diff \bw \nonumber\\
                         &= \frac{1}{Z} \, e^{-\gamma \Rl(\bwstarl) - \frac{\gamma}{6} \ve(r)} \,
                           \frac{1}{\sqrt{\det(\bHstarl)}} \int_{\sB(r)} e^{- \frac{\gamma}{2} \|\bu\|^2} \diff \bu \nonumber\\
                         &= \frac{1}{Z} \, e^{-\gamma \Rl(\bwstarl) - \frac{\gamma}{6} \ve(r)} \,
                           \frac{1}{\sqrt{\det(\bHstarl)}} \pr{\frac{2 \pi}{\gamma}}^{\frac{d}{2}} P\pr{\frac{d}{2}, \frac{r^2 \gamma}{2}}~.
                           \label{eq:P_ellipse_bounds_gauss_int_lower}
\end{align}
Finally, we give a lower bound on the probability of $\sEstar(r)$.
We start by upper bounding the normalization constant using the lower Taylor expansion~\eqref{eq:taylor_lower},
\begin{align*}
  Z &= \int_{\hypsp} e^{-\gamma \Rl(\bw)} \diff \bw\\
    &\leq e^{-\gamma \Rl(\bwstarl) + \frac{\gamma}{6} \ve(r)} \int_{\hypsp} e^{- \frac{\gamma}{2} \|\bw - \bwstarl\|_{\bHstarl}^2} \diff \bw\\
    &\leq e^{-\gamma \Rl(\bwstarl) + \frac{\gamma}{6} \ve(r)} \, \frac{1}{\sqrt{\det(\bHstarl)}} \pr{\frac{2 \pi}{\gamma}}^{\frac{d}{2}}~.
\end{align*}
Combining the above with~\eqref{eq:P_ellipse_bounds_gauss_int_lower} gives
$
  \P\big(\sEstar(r)\big) \geq e^{-\frac{\gamma}{3} \ve(r)} P\pr{\frac{d}{2}, \frac{r^2 \gamma}{2}}
$
thus completing the proof.
\end{proof}
\paragraph{Lemma~\ref{lem:erm_prob} (restated)}
  For all $r > 0$,
  \begin{align*}
    \pigamma(i)
    \leq
    \frac{e^{\frac{\gamma}{3} \max_{k \in \sI}\ve_k(r)}}
    {\sum_{j \in \sI} e^{\gamma \pr{ \Rl(\bwstarli) - \Rl(\bwstarlj)}}
    \sqrt{\frac{\det(\bHstarli)}{\det(\bHstarlj)} }
    }
    \qquad i \in \sI~.
  \end{align*}
  Moreover, assuming without loss of generality that $\Rl(\bwstarli) = 0$ for all $i \in \sI\glob$, and setting $r = \gamma^{-p}$ for $p > 0$, we have
  \[
    \lim_{\gamma \rightarrow \infty}\pigamma(i)
    = \begin{cases}
    \frac{1}{\sum_{j \in \sI\glob} \sqrt{\frac{\det(\bHstar_{\lambda, i})}{\det(\bHstar_{\lambda, j})} }} & i \in \sI\glob
  \\
    0 & i \in \sI\subopt~.
  \end{cases}
  \]
\begin{proof}
  Throughout this proof we consider probabilities of ellipsoids under the density $e^{-\gamma \Rl(\bw)} / Z$, and we abbreviate $\Rstarli  = \Rl(\bwstarli)$.
  Applying Lemma~\ref{lem:ellipse_prob} with $\bwstarl = \bwstarli$ readily gives
\begin{align}  
  \pigamma(i)
  &=
    \frac{\P(\bw \in \sEstar_i(r))}{\sum_{j \in \sI} \P(\bw \in \sEstar_j(r))} \nonumber\\
  &\leq
    \frac{e^{\frac{\gamma}{6} \ve_i(r)}}
    {\sum_{j \in \sI} e^{\gamma \pr{ \Rstarli - \Rstarlj} - \frac{\gamma}{6} \ve_j(r)}
    \sqrt{\frac{\det(\bHstarli)}{\det(\bHstarlj)} }
    } \nonumber\\
  &\leq
    \frac{e^{\frac{\gamma}{6} \max_{k \in \sI}\ve_k(r)}}
    {\sum_{j \in \sI} e^{\gamma \pr{ \Rstarli - \Rstarlj} - \frac{\gamma}{6} \max_{k \in \sI}\ve_k(r)}
    \sqrt{\frac{\det(\bHstarli)}{\det(\bHstarlj)} }
    } \nonumber\\
  &=
    \frac{e^{\frac{\gamma}{3} \max_{k \in \sI}\ve_k(r)}}
    {\sum_{j \in \sI} e^{\gamma \pr{ \Rstarli - \Rstarlj}}
    \sqrt{\frac{\det(\bHstarli)}{\det(\bHstarlj)} }
    }~.
  \label{eq:pi_gamma_bound_proof}
\end{align}
This proves the first statement.

Now we look at the asymptotics of $\pigamma(i)$ as $\gamma \to \infty$ assuming that $r = \gamma^{-p}$ for $p > 0$.
First, observe that for any $i \in \sI$
\begin{align*}
  \lim_{\gamma \rightarrow \infty} \ve_i(\gamma^{-p})
  &=
    \lim_{\gamma \rightarrow \infty} \Lipi(\gamma^{-p}) \pr{\frac{1}{\gamma^p \sqrt{\lmini + \lambda}}}^3
  = 0
\end{align*}
because $\lim_{\gamma \to \infty} \Lipi(\gamma^{-p}) = \scO(1)$ and $\lmini + \lambda > 0$.
Thus, the numerator of~\eqref{eq:pi_gamma_bound_proof} approaches $1$.
Now, we turn our attention to the denominator. First, we consider global minimizers recalling our assumption that $\Rstarli = 0$.
Denoting $\delta_{i,j}(\gamma) = e^{\gamma \pr{ \Rstarli - \Rstarlj}}$, we observe that for all $\gamma \geq 0$ and $i \in \sI\glob$,
\begin{align*}
\lim_{\gamma \rightarrow \infty}  \delta_{i,j}(\gamma) =
\begin{cases}
 1 & j \in \sI\glob
\\
 0 & j \in \sI\subopt
\end{cases}
\end{align*}
where the second case holds because the exponent in $\delta_{i,j}(\gamma)$ is negative.
This implies
\[
  \lim_{\gamma \rightarrow \infty}\pigamma(i) \leq
  \frac{1}{\sum_{j \in \sI\glob} \sqrt{\frac{\det(\bHstar_{\lambda, i})}{\det(\bHstar_{\lambda, j})} }}
  \qquad i \in \sI\glob.
\]
Next, we consider the local minima, and observe that for all $\gamma \geq 0$ and $i \in \sI\subopt$,
\begin{align*}
\lim_{\gamma \rightarrow \infty}  \delta_{i,j}(\gamma) =
\begin{cases}
 0 & \text{if $\Rstarli \leq \Rstarlj$}
\\
 \infty & \text{otherwise.}
\end{cases}
\end{align*}
Therefore, for all $i \in \sI\subopt$,
$
  \lim_{\gamma \rightarrow \infty}\pigamma(i) = 0
$.
This proves the second statement and completes the proof.
\end{proof}
\subsection{Global Excess Risk Bounds}
We first show a nonasymptotic (i.e., finite $\gamma$) global excess risk bound.
\paragraph{Theorem~\ref{thm:global_excess_risk} (restated)}
  Assume the same as in Lemma~\ref{lem:truncated_EER_to_ERM}.
  Then for any $r \in [0, r_0]$ the global excess risk satisfies
  \begin{equation*}
  \excess(\pigamma)
  \leqC
  \frac{1}{\gamma} \, \E\br{\tr\pr{\bHstarI \bHstarlIinv}}
  + \frac{\gamma}{\sqrt{m}}
  +  \E[\ve_I(r)] + \sqrt{\gamma \E[\ve_I(r)]} + \Pgamma(\mincomp(r))
  \end{equation*}
  where the expectation is taken with respect to $I \sim \pigamma$
  and the probability of the complement of the minima is bounded as
  \begin{align}
    \Pgamma(\mincomp(r))
    &\leq
      1 - \pr{1 - d e^{-r^2 \gamma \alpha_{d/2}}} \sum_{i \in \sI} e^{-\frac{1}{3} \gamma \ve_i(r)}
  \end{align}
  with $\alpha_{d/2}$ defined in~\eqref{eq:alpha_a}.
\begin{proof}
Denote the sample-dependent global excess risk by
\begin{align*}
  \excess_S(\pigamma) = \E_{\bw \sim \ps}\br{ R(\bw) } - \E_{I \sim \pigamma}\br{ R(\bwstarlI) }
\end{align*}
and let the probabilities $\P\big(\sEstar_i(r)\big)$ and $\P\big(\mincomp(r)\big)$ be defined with respect to the population Gibbs distribution
$
  p_{\gamma}(\bw) \propto e^{-\gamma \Rl(\bw)}
$ with $\gamma > 0$.

We first focus on the first term on the right-hand side of $\excess_S(\pigamma)$. By the law of total expectation, for any $r \in [0, r_0]$,
\begin{align}
  \E_{\bw \sim \ps}\br{ R(\bw) } &= \sum_{i \in \sI} \P\big(\sEstar_i(r)\big) \E\br{R(\bw) \,\big|\, \bw \in \sEstar_i(r)}
                                   + \P\big(\mincomp(r)\big) \E\br{R(\bw) \,\big|\, \bw \in \mincomp(r)} \nonumber\\
                               &\leq \sum_{i \in \sI} \frac{\P\big(\sEstar_i(r)\big)}{\sum_{j \in \sI} \P\big(\sEstar_j(r)\big)} \, \E\br{R(\bw) \,\big|\, \bw \in \sEstar_i(r)} \tag{ellipsoids are disjoint by Assumption~\ref{asm:global}}\\
                                 &+ \P\big(\mincomp(r)\big) M \tag{risk is bounded}\\
                               &= \sum_{i \in \sI} \pigamma(I=i) \E\br{R(\bw) \,\big|\, \bw \in \sEstar_i(r)} \tag{by definition of $\pigamma$} \nonumber\\
                               &+ \P\big(\mincomp(r)\big) M~. \label{eq:half_bound_on_global_excess_risk}
\end{align}
An upper bound on $\E\br{R(\bw) \,\big|\, \bw \in \sEstar_i(r)}$ is given by Theorem~\ref{thm:gibbs_local_excess}, thus all that is left to show is that the probability of the complement is small. Since the ellipsoids are disjoint
\[
  \P\big(\mincomp(r)\big) = 1 - \sum_{i \in \sI} \P\big(\sEstar_i(r)\big)~.
\]
To upper bound $\P\big(\mincomp(r)\big)$ we need a lower bound on $\P\big(\sEstar(r)\big)$. This is provided by the last inequality in Lemma~\ref{lem:ellipse_prob}, that is,
\begin{align}
\nonumber
  \sum_{i \in \sI} \P\big(\sEstar_i(r)\big)
  &\geq
    P\pr{\frac{d}{2}, \frac{r^2 \gamma}{2}} \sum_{i \in \sI} e^{-\frac{\gamma \ve_i(r)}{3}}\\
  &\geq
    \pr{1 - e^{-r^2 \gamma \alpha_{d/2}}}^{\frac{d}{2}} \sum_{i \in \sI} e^{-\frac{\gamma \ve_i(r)}{3}} \label{eq:lower_bound_on_sum_P_Estar_2}\\  
  &\geq
  \nonumber
    \pr{1 - d e^{-r^2 \gamma \alpha_{d/2}}} \sum_{i \in \sI} e^{-\frac{\gamma \ve_i(r)}{3}}
\end{align}
where~\eqref{eq:lower_bound_on_sum_P_Estar_2} is derived from the lower bound on the regularized Gamma function~\eqref{eq:lb_regularized_gamma},
and the last inequality is obtained from the Bernoulli inequality
\[
  (1 + x)^{\frac{d}{2}} \geq (1 + x)^d \geq 1 + d x \qquad d \in \mathbb{N}, \ x \geq -1~.
\]
Thus,
\begin{align}
  \P\big(\mincomp(r)\big) \leq 1 - \pr{1 - d e^{-r^2 \gamma \alpha_{d/2}}} \sum_{i \in \sI} e^{-\frac{1}{3} \gamma \ve_i(r)}~.
\end{align}
Taking expectation with respect to $S$, and combining Theorem~\ref{thm:gibbs_local_excess} with~\eqref{eq:half_bound_on_global_excess_risk} and Jensen's inequality, we obtain
\begin{align*}
  \excess(\pigamma)
  &\leq
    \E_{I \sim \pigamma}\br{ \E_S\Big[\E\br{R(\bw) \,\big|\, \bw \in \sEstar_i(r)} \Big] - R(\bwstarlI) }
    + M\, \P\big(\mincomp(r)\big) \\
  &\leq
    \frac{1}{\gamma} \, \E_{I \sim \pigamma}\br{\tr\pr{\bHstarI \bHstarlIinv}} + \frac{1}{6}\E_{I \sim \pigamma}\br{\ve_I(r)}
    + \frac{M}{2} \sqrt{\frac{\gamma}{3} \E_{I \sim \pigamma}\br{\ve_I(r)} + \frac{M^2 \gamma^2}{2 m}}
    + \frac{M^2 \gamma}{2 m} \\
    &+ M\, \P\big(\mincomp(r)\big)~.
\end{align*}
The proof is concluded by stating the above with respect to radius $r \in [0, r_0]$ ---recall that the radius cannot exceed $r_0$, the largest radius ensuring that ellipsoids remain disjoint.
\end{proof}
\paragraph{Corollary~\ref{cor:global_excess_risk_asymptotic} (restated)}
  Assume the same as in Lemma~\ref{lem:truncated_EER_to_ERM}.
  Then, for any $r > 0$, the global asymptotic pseudo-excess risk satisfies
  \[
    \excess^{\infty}_r \leqC \frac{1}{\gamma} \, \E\br{\tr\pr{\bHstarI \bHstarlIinv}}
    + \E \br{\ve_I(r)}
    + \sqrt{\gamma \E \br{\ve_I(r)} + \frac{\gamma^2}{m}}
  + \frac{\gamma}{m}
\]
where $I$ is distributed according to
\[
\piinfty(i) =
  \frac{1}{\sum_{j \in \sI\glob} \sqrt{\frac{\det(\bHstar_{\lambda, i})}{\det(\bHstar_{\lambda, j})} }}~.
\]
\begin{proof}
  Recall that the global asymptotic pseudo-excess risk is defined as
  \begin{align*}
    \excess^{\infty}_r = \E_{I \sim \piinfty}\br{\E_S\br{\E_{\bw \sim \ps}\br{R(\bw) \,\big|\, \bw \in \sEstar_I(r)}} - R(\bwstarlI)}~.
  \end{align*}
  Distribution $\piinfty$ is given by Lemma~\ref{lem:erm_prob}, while the local excess risk centered at $\bwstarlI$ is bounded by Theorem~\ref{thm:gibbs_local_excess}.
  This immediately yields the statement. 
\end{proof}

\subsubsection*{Acknowledgments}
Authors would like to thank Olivier Bousquet, Sébastien Gerchinovitz, and Abbas Mehrabian for stimulating discussions on this work.

\bibliography{learning}
\end{document}